\documentclass[10pt,letterpaper]{article}
\usepackage[utf8]{inputenc} 
\usepackage[T1]{fontenc}    
\usepackage[margin=1in]{geometry}

\usepackage[colorlinks=true,linkcolor=blue,citecolor=blue]{hyperref}
\usepackage{amsfonts}
\usepackage{amsthm}
\usepackage{mathtools}
\usepackage{amssymb}
\usepackage{thmtools}
\usepackage{thm-restate} 
\usepackage{cleveref}
\usepackage{float} 
\usepackage{soul}
\usepackage{nicefrac}
\usepackage{url}           
\usepackage{booktabs}  
\usepackage{subcaption}

\usepackage[lined,boxed,ruled,norelsize,algo2e,linesnumbered]{algorithm2e}
\usepackage{algorithm}
\usepackage{algorithmic}

\newtheorem{definition}{Definition}

\theoremstyle{definition}
\newtheorem{remark}{Remark}

\usepackage{xcolor}
\usepackage{bbm}
\allowdisplaybreaks

\makeatletter
\newcommand*{\rom}[1]{\expandafter\@slowromancap\romannumeral #1@}
\makeatother

\title{Machine-Learned Sampling of Conditioned Path Measures}

\author{
Qijia Jiang\thanks{UC Davis, \texttt{qjang@ucdavis.edu}}
\and
Reuben Cohn-Gordon\thanks{UC Berkeley, \texttt{reubenharry@gmail.com}}
}  
\date{\today}

\begin{document}

\maketitle

\begin{abstract}
We propose algorithms for sampling from posterior path measures $\mathcal{P}(\mathcal{C}([0, T], \mathbb{R}^d))$ under a general prior process. This leverages ideas from (1) controlled equilibrium dynamics, which gradually transport between two path measures, and (2) optimization in $\infty$-dimensional probability space endowed with a Wasserstein metric, which can be used to evolve a density curve under the specified likelihood. The resulting algorithms are theoretically grounded and can be integrated seamlessly with neural networks for learning the target trajectory ensembles, without access to data.
\end{abstract}

% often reduce to learning a drift term
% gradient flow

\section{Introduction}
\label{sec:framework}

Simulating trajectories through a space, according to a given distribution, is a problem which appears in a broad range of scientific applications.
A general formulation of the problem is given by considering a prior path measure that admits a representation as a stochastic differential equation (SDE):
% We start with the most general setup for the problem under consideration: given a 
\begin{equation}
\label{eqn:prior_sde}
dX_t=u_t^{\text{ref}}(X_t)dt+\sqrt{2}dW_t, X_0\sim\rho_0  \quad \text{Denote this path measure as } (\mathbb{P}_t^{\text{ref}})_{t\leq T}\, ,
\end{equation}
where the $u_t^{\text{ref}}$ part of the drift is assumed known (e.g., with $u_t^{\mathit{ref}}=-\nabla V_t$, for $V_t$ a double-well potential). Suppose we then observe, at multiple time points $t_1,\dots,t_K$, potentially noisy observations of the dynamics of the form 
\[y_k=x_{t_k}+\sigma z_k, \quad k=1,\dots, K\,,\]
% were $z_k$ are drawn from standard normal distributions, then our goal is to 
which is an additive Gaussian noise model. Given such a prior and data $\{y_k\}_{k}$, the goal is to
% : given the data $\{y_k\}_{k}$ at multiple time points, and a prior on the process as in \eqref{eqn:prior_sde}, 
infer the posterior path measure $q$, which is known to take the form \cite{raginsky2024variational}:
% and/or the corresponding drift $b$:
\begin{equation}
\label{eqn:posterior_sde}
dX_t= [u_t^{\text{ref}}(X_t)+b_t(X_t)] dt+\sqrt{2}dW_t, X_0\sim\rho_0\quad \text{Denote this path measure as } (q_t)_{t\leq T}\, .
\end{equation}
Equivalently, given methods to sample from $(\mathbb{P}_t^{\text{ref}})_{t\leq T}$, we wish to design algorithms to sample from $(q_t)_{t\leq T}$. Such problems find applications in trajectory inference in biology \cite{chizat2022trajectory}, time series Bayesian analysis, among many others. Notably, this is a generalization of Transition Path Sampling (TPS) from chemistry \cite{bolhuis2021transition}, since one can simply take $t_K=T$ to be a single observation at time $t=T$ with $y_K=B$, $\rho_0=\delta_A$ and $\sigma$ small, i.e., we aim to sample from the ``bridge" $\mathbb{P}(x_{T-h},x_{T-2h},\dots,x_h\vert x_0=A, x_T=B)$ in TPS with $J(x;y)\approx \delta(x_T\neq B)$.

The fact that under mild regularity conditions, the posterior path measure (conditioned on the observations $\{y_k\}$) can always be represented with an additional drift $b_t^\theta$ \cite{raginsky2024variational} means that it can be understood as a ``controlled diffusion process''. Thus, the task amounts to inferring the drift of the posterior SDE given a prior, assuming both \eqref{eqn:prior_sde}, \eqref{eqn:posterior_sde} have the same initial $\rho_0$ that is perfectly known.

Mathematically, we write the conditional density in path space for $\{x(t)\}_{t\leq T}$ as \cite[Section 4.1]{apte2007sampling}
 \begin{equation}
 \label{eqn:Q_star}
 Q(x) \propto \rho_0(x_{t_0})\exp(-I(x)-J(x;y))
 \end{equation}
 where $I(x)$ is the Onsager-Machlup action functional for the unconditional density of path
 \[I(x)=\frac{1}{2}\int_0^T \left(\frac{1}{2}\left\|\frac{dX_t}{dt}-u_t^{\text{ref}}(X_t)\right\|_2^2+\frac{1}{2}\nabla\cdot u_t^{\text{ref}}(X_t)\right) dt\]
 % this provides an expression for the probability of paths
and $J(x;y)$ is the likelihood term
\begin{equation}
\label{eqn:J_functional}
J(x;y)=\frac{1}{2\sigma^2}\sum_{k=1}^K \|y_k-x_{t_k}\|_2^2=\frac{1}{2\sigma^2}\int_0^T \sum_{k=1}^K \|y_k-x_t\|_2^2 \cdot \delta(t-t_k)\,  dt =:\int_0^T J_t(x_t)\, dt\, .
\end{equation}
Above, the observations $y$ parameterize the posterior $Q(x)$ we are interested in sampling from, and it is a \emph{global} conditioning. We assume that (1) the observation at time $t_k$, $y_k$, depends only on $x_{t_k}$, and not other time points; (2) $(\mathbb{P}_t^{\text{ref}})_{t\leq T}$ in \eqref{eqn:prior_sde} is a Markov process throughout the paper. Note that the resulting quantity of interest here is not a deterministic trajectory but rather a probability measure over path ensembles. Maximizing the functional $Q(x)$ will give the most probable path under observations. %Relation to filtering / smoothing.

\paragraph{Contributions} 
In this work, we are interested in identifying the optimal $b_t^*$, which provides a convenient way to draw additional samples / infer the trajectory, rather than just producing a set of sampled trajectories from the target. 
In particular, we explore novel ways to adapt recent ML ideas from $\mathcal{P}(\mathbb{R}^d)$ sampling to $\mathcal{P}(\mathcal{C}([0,T],\mathbb{R}^d))$ sampling. Our two proposals are presented in Section \ref{sec:controlled_path_measure} and \ref{sec:WGF}. Although both methods rely on the posterior nature of the measure rather than just sampling from a generic path measure, they take different viewpoints on the problem. Thus, we expect them to perform well under different likelihood models, and are suitable for complementary use cases. Experimental results can be found in Appendix~\ref{sec:numerics}.  
% Building on recent work focused on neural network based approaches to sampling from $\mathbb{R}^n$, we propose two methods to sample on paths which take advantage of the fact that this amounts to learning a controlled diffusion.

%\vspace{-0.4cm}
\begin{figure}[H]
\label{fig:paths}
    \centering
    \includegraphics[width=0.3\linewidth]{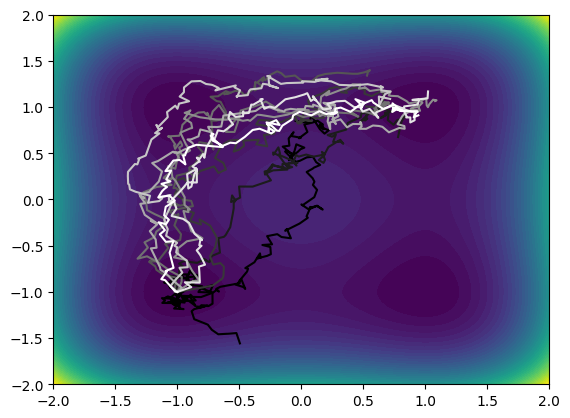}
    \includegraphics[width=0.3\linewidth]{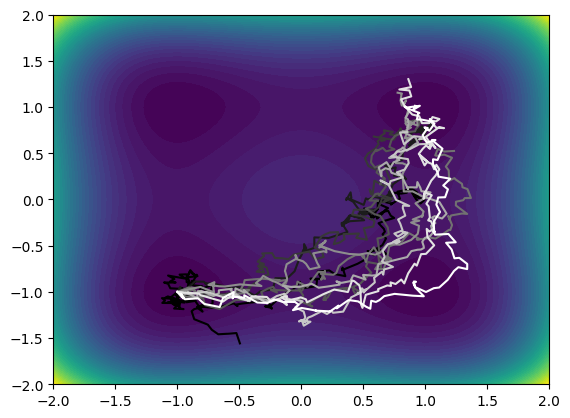}
    \caption{On a potential landscape $V(x) = 5(x_1^2-1)^2 + 5(x_2^2-1)^2$, we condition on the final point of the path being in the top right potential well, and the midpoint being in the bottom right well or top left well. We use Algorithm~\ref{alg:transport_b} from Section \ref{sec:controlled_path_measure}, varying $s$ from $0.0$ (fully black path) to $1.0$ (fully white path) in increments of $0.1$.}
\end{figure}
%\vspace{-0.4cm}

% and fit on the order of $T$ score-matching diffusion models. $\infty$-dimensional 

% Upon reaching an optimal ρt, our primary object of interest is the dynamics model corresponding to ρ˙t and parameterized by the optimal sρ˙t, which may be used to transport particles or predict individual trajectories.

% We summarize our contribution as follows:
% \begin{itemize}
% \item For a general Bayesian posterior sampling over the path space, we 
% \item Compared to some of the SB-specific approaches mentioned in Section~\ref{sec:previous_work}, our framework is more general 
% \end{itemize}

\section{Related Work}
\label{sec:previous_work}
With regards to the infinite-dimensional path-space sampling problem on $\mathcal{C}([0,T],\mathbb{R}^d)$, there exists a longstanding body of research concerning at least two different types of methods.

\paragraph{Existing Method 1: Langevin Stochastic PDE} With the target $Q(x)$ defined as above, one can write down a Langevin SPDE on $x(t,s)$:
\begin{equation}
\label{eqn:langevin-spde}
\frac{\partial x(t,s)}{\partial s}=\delta_x \log Q(x)+\sqrt{2}\frac{\partial W}{\partial s} \quad \text{for} \;(s,t)\in (0,\infty)\times (0,T)
\end{equation}
where $s$ is the algorithmic time (we expect as $s\rightarrow\infty$ relying on ergodicity of the Markov chain it should converge to the desired equilibrium $Q(x)$) and $t\in[0,T]$ is the physical time that simply index the path. The variational derivative $\delta_x \log Q(x)$ can be worked out analytically (as done in \cite{apte2007sampling}). Together with appropriate boundary conditions, this can serve as a sampler for the problem (assuming observation at endpoint $T$):  % -(\frac{1}{2}\nabla f(x)-\frac{1}{2}\nabla f(x)^\top) \frac{dx_t}{dt}
%\begin{equation}
%\begin{aligned}
\begin{align}
\delta_x \log Q(x) &= \frac{1}{2}\frac{d^2 x_t}{dt^2}-\frac{1}{2}u(x)\nabla u(x)-\frac{1}{2}\nabla (\nabla\cdot u(x))+\sum_{k=1}^K \frac{1}{\sigma^2} (y_k-x(t_k,\cdot)) \nonumber\\
\frac{\partial x(0,s)}{\partial t}&-u(x(0,s))+2\nabla \log \rho_0(x(0,s))=0,\quad t=0 \label{eqn:spde_general}\\
\frac{\partial x(T,s)}{\partial t}&-u(x(T,s))-\frac{2}{\sigma^2}[y_K-x(T,s)]=0, \quad t=T \nonumber
\end{align}
%\end{aligned}
%\end{equation}
There have been attempts to pursue different variants of MCMC methods from Stuart et al. \cite{stuart2004conditional, beskos2008mcmc} (possibly with MH adjustment), with a connection to data assimilation generally.

\paragraph{Existing Method 2: learning $b_t(X_t)$} Rather than working directly with paths, which can be memory-intensive, another approach is to directly learn the additional drift $b_t(X_t)$ in \eqref{eqn:posterior_sde}, which in fact takes a gradient form. This often reduces to an optimal control problem or alternatively a high-dimensional PDE problem. It is known, since the posterior is a reweighting of the prior path measure, the drift admits closed-form in a path integral representation \cite[Theorem 1]{raginsky2024variational}:
\begin{equation}
\label{eqn:path_integral_control}
b_t(z)=2\nabla \log \mathbb{E}_{x\sim\mathbb{P}^{\text{ref}}}\left[\exp\left(-\int_t^T J_s(x_s) ds\right) \Big| x_t=z\right] =:-\nabla v_t(z)
\end{equation}
for the conditioned SDE, where the expectation is taken over the prior SDE with drift $u_t^{\text{ref}}$. This is something one may hope to simulate to estimate and recovers Doob's $h$-transform $2\nabla_x\log p_{T|t}(X_T=B|X_t=x)$ when we only observe at time $T$ \cite{pidstrigach2025}. But since the quantity under estimation is typically a rare-event under the reference measure, even with importance weights, it is generally hard to do well; many paths will have negligible weights and the relevant paths are exponentially rare. There is a PDE one can write down for $v_t$ in \eqref{eqn:path_integral_control} as well using Feynman-Kac:
\begin{equation}
\label{eqn:HJB}
\frac{\partial}{\partial t} v_t(z) + u_t^{\text{ref}}(z)^{\top} \nabla v_t(z)+\Delta v_t(z)+J_t(z)= \|\nabla v_t(z)\|^2
\end{equation}   % we don't necessarily have terminal constraint g
but it isn't clear that solving a high-dimensional PDE would be any easier. Some other related work along these lines (with possibly marginal variational approximation) include \cite{holdijk2024stochastic, du2024doob, sutter2016variational}. A method based on score-matching for learning this drift has also been proposed \cite{heng2021simulating}.
%turns it into stochastic optimal control problem with appropriate importance weights on the loss/gradient but I don't think it's exactly solving ...

In the case of TPS, this is a bona-fide Schr\"odinger bridge (SB) problem with two delta terminal constraints, for which one can do path-space Sinkhorn / IPF by iteratively minimizing $\text{KL}(\overrightarrow{\mathbb{P}}^{\delta_A, u}\Vert\overleftarrow{\mathbb{P}}^{\delta_B, v})$ over the two slots alternatively. It amounts to setting $u_0=u^{\text{ref}}$, and perform a sequence of score-matching time reversals to learn the drifts $u^*,v^*$. However, this is not as general as our setup, and methodology from \cite{peluchetti2023diffusion} that is also based on an alternating projection for SB is difficult to generalize when $u^{\text{ref}}\neq 0$ since simulating from the conditional bridge would be challenging. We note that compared to the series of work on flow matching \cite{tong2023improving}, which aim to match the time marginals for a prescribed path (often set up as mixtures of diffusion processes), we aim for the more ambitious goal of sampling from the true target path measure.

Crucially, our setup does not assume access to data, which distinguishes it from the proposals in \cite{triplett2025diffusion} that leverage diffusion model for conditional sampling from a function-valued probability density over $\{X_t\}_t$.

% \begin{remark}
% % Match mixture of diffusions via Doob's $h$-transform, but \cite{peluchetti2023diffusion} and related flow matching methods only match the time marginals. 
% Denoising Diffusion is another option for this problem: need to be able to sample from the likelihood and the prior to implement \cite[Section 4.6]{lim2023score}. 
% \end{remark}

% We expect the second approach may have difficulty dealing with jumps in the likelihood term $J$ in the sparse observation setting

\section{Controlled Transport from Prior to Posterior}
\label{sec:controlled_path_measure}

\subsection{Tracking $b_t$}
\label{sec:track_b_transport}

Building on \cite{nusken2024stein} and \cite{albergo2024nets}, we introduce a sequence of path measures as
    \[\pi_s(x) = Z_s^{-1}\cdot(\rho_0(x_{t_0})\exp(-I(x)-s\cdot J(x;y)))\quad s\in[0,1]\]
    where $Z_s=\int \rho_0(x_{t_0})\exp(-I(x)-s\cdot J(x;y)) dx$. It is clear from the definition that $\pi_s$ interpolate between the prior and the posterior. Similarly, we introduce a family of SDEs parametrized as 
    \[dX_t^s = b_t^s(X_t^s)dt +\sqrt{2}dW_t,\; X_0^s\sim\rho_0 \;\; \text{and}\;\; b_t^0=u_t^{\text{ref}}\]
    inducing corresponding path measures 
    \[\rho_s(x)=Y_s^{-1}\cdot\left[\rho_0(x_{t_0})\exp\left(-\frac{1}{2}\int_0^T \left(\frac{1}{2}\left\|\frac{dx_t}{dt}-b_t^s(x_t)\right\|_2^2+\frac{1}{2}\nabla\cdot b_t^s(x_t)\right) dt\right)\right]\, .\]
    If we can learn the drift change in $b$ at each step $s$ as the solution of (holding for all path $x$)
    \begin{align}
    &\frac{\partial \pi_s(x)}{\partial s}=\pi_s(x)\left(-J(x;y)+\mathbb{E}_{\pi_s(x)}[J(x;y)]\right)\label{eqn:FR}\\
    &\overset{!}{=} \frac{\partial \rho_s(x)}{\partial s}=\rho_s(x)\left(h_s(x)-\mathbb{E}_{\rho_s(x)}[h_s(x)]\right)\nonumber\\
    &\text{where } h_s(x)=-\frac{1}{2}\int_0^T \left(b_t^s(x_t)-\frac{dx_t}{dt}\right)^\top \frac{\partial b_t^s}{\partial s}(x_t)+\frac{1}{2}\nabla\cdot \frac{\partial b_t^s}{\partial s}(x_t)\, dt\nonumber
    \end{align}
    then the resulting drift $\{b_t^s\}_t$ will follow the path measure $\pi_s(x)$ for all $s\in[0,1]$ -- thus \emph{exactly} hitting target at time $s=1$. Unlike MCMC \eqref{eqn:langevin-spde}, this is ``controlled", which means we don't need to run until $s=\infty$.
Operationally, one starts by drawing samples from the prior $\pi_0(x)$, estimate the expectation in 
\begin{equation}
\label{eqn:consistency}
h_s(x)-\mathbb{E}_{\rho_s(x)}[h_s(x)]=-J(x;y)+\mathbb{E}_{\pi_s(x)}[J(x;y)]
\end{equation}
using empirical samples, and solve for $\partial b^s/\partial s$ on the LHS to update the drift. This is followed by re-simulate paths using this new drift 
\[b^{s+\delta s} \approx b^s+(\partial b^s/\partial s)\cdot \delta s\] 
for again imposing the consistency equation $\pi_{s+\delta_s}(x)=\rho_{s+\delta_s}(x)$ in \eqref{eqn:consistency} and repeat until $s\approx 1$ since $\pi_1(x)=Q(x)$. Note that estimating \eqref{eqn:consistency} does not require knowing the normalizing constant. In practice, one can parametrize $\{\partial b_t^s/\partial s\}_t$ as a Neural Network taking $(x_t^s,t,s)\in \mathbb{R}^d\times [0,T]\times [0,1]$ as input and outputting a $\mathbb{R}^d$-valued vector.

To correct for the possible lag $\rho_s\neq \pi_s$, one may optionally simulate the trajectories using the current drift $b^s$, followed by a few steps of MCMC w.r.t the density $\pi_s$, after which we learn the update $\partial b^s/\partial s$ using the latest trajectories, and repeat -- this is somewhat similar to Algorithm 2 from \cite{nusken2024stein} where SVGD is used in between for $\mathcal{P}(\mathbb{R}^d)$ sampling.

The intuition behind this approach is that it is easier to conceptualize the algorithmic dynamics on the path measure space $q^{n}\rightarrow q^{n+1}\rightarrow\dots \rightarrow q^*$ and work out the corresponding drift update for the SDE $b^n\rightarrow b^{n+1}\rightarrow\dots\rightarrow b^*$ for implementation purpose. At a high level, this is a transport idea in the path measure setting. 
%
% \cite{stoltz2010free} gave wonderful exposition on the subject of free-energy calculation. 
\paragraph{Path-space Normalizing Constant Estimator} With the $\rho_s=\pi_s$ interpolation method above (i.e., maintaining equilibrium dynamics w.r.t $\pi_s\; \forall s\in[0,1]$), one can consider thermodynamic integration in path space and calculate the partition function for $Q$ as
\begin{align*}
\frac{\partial Z_s}{\partial s}&=\int -J(x;y)\rho_0(x_{t_0})\exp(-I(x)-s\cdot J(x;y)) dx\\
&= \frac{\int -J(x;y)\rho_0(x_{t_0})\exp(-I(x)-s\cdot J(x;y)) dx}{\int \rho_0(x_{t_0})\exp(-I(x)-s\cdot J(x;y)) dx} Z_s=\mathbb{E}_{\pi_s}[-J(x;y)] \cdot Z_s
\end{align*}
therefore $Z_T=Z_0\cdot \exp\left(\int_0^T \mathbb{E}_{\pi_s}[-J(x;y)]\, ds\right)$ assuming the initial $Z_0$ is analytically tractable. Quantity of such type is often of interest in chemistry applications. The end-to-end method is summarized in Algorithm \ref{alg:transport_b} (c.f. Appendix~\ref{app:pseudocode}).

\subsection{Kernel Method}
\label{sec:kernel}

% We considered in \eqref{eqn:FR} in some sense Fisher-Rao for designing the $q$-space dynamics but others can be considered as well. 

Without NN, we can solve \eqref{eqn:consistency} with Kernel Ridge Regression (KRR) by assuming $b$ belongs to a RKHS $\mathcal{H}_k$. The benefit of this is that everything now becomes closed-form. For each $s$, denote $v^s(x_t,t) := \frac{\partial b^s}{\partial s}(x_t,t)$ as a vector field mapping $\mathbb{R}^d\times\mathbb{R}$ to $\mathbb{R}^d$, where we have an extra $t\in[T]$ dimension compared to \cite{nusken2024stein} because we are working in path-space. Introduce the linear operator $\mathcal{S}$ mapping the function $v^s$ and path $x^i\in\mathcal{C}([0,T],\mathbb{R}^d)$ to a scalar $h_s(x^i)-\frac{1}{N}\sum_{j=1}^N h_s(x^j)$. Assuming $N$ trajectories, at each $s\in [0,1]$, one solves
\begin{equation}
\label{eqn:KRR}
\arg\min_{v\in \mathcal{H}_k}\; \frac{1}{N}\sum_{i=1}^N (\mathcal{S}(v^s)(x^i)-j(x^i))^2+\lambda\|v^s\|_{\mathcal{H}_k}^2
\end{equation}
by letting $j(x^i):=-J(x^i)+\frac{1}{N}\sum_{j=1}^N J(x^j)$ that takes a path $x^i\in\mathcal{C}([0,T],\mathbb{R}^d)$ as input and outputs a scalar.
\begin{restatable}[Kernel Method]{lemma}{KRR}
\label{lem:KRR}
KRR method for solving \eqref{eqn:consistency} with empirical samples, written in \eqref{eqn:KRR} for each $s \in[0,1]$, can be implemented analytically by solving a linear system of size $N\times N$ at each $t \in [T]$.
\end{restatable}
\begin{remark}
The way we wrote it, it may seem that the $T$ problems to solve are independent from each other, but they are in fact coupled through terms like $dx_t^i/dt$ in the proof. 
\end{remark}

\subsection{Equilibrium vs. Non-Equilibrium Thermodynamics}
\label{sec:Non-Equilibrium}
Our methods from Section \ref{sec:track_b_transport}-\ref{sec:kernel},  and the alternative ``path-based drifts" discussed in Appendix \ref{app:control_methods}, being based on having the dynamics exactly following the interpolation curve, are equilibrium methods. Alternatively, one can use the annealed gradient $\delta_x \log \pi_s(x)$ from the SPDE, together with MH adjustment, as a time-inhomogeneous transition kernel, followed by \emph{weighting} on the samples to track the $\pi_s$ curve (Algorithm \ref{alg:spde}). Since the MCMC dynamics alone does not impose $\rho_s=\pi_s$ exactly (a single step of MCMC with changing target will result in a delay), these fall in the category of non-equilibrium methods. The result below bears a resemblance to that in \cite[Section 4]{stoltz2007path}, and is a generalization of Jarzynski's equality to path space.

\begin{restatable}[Time-discrete Jarzynski]{lemma}{Jarzynski}
\label{lem:Jarzynski}
The analogue of the Jarzynski equality in discrete time in path space is: for $s\in [0,1]$ along the Algorithm~\ref{alg:spde} dynamics
    \begin{align*}
    \frac{Z_s}{Z_0} &=\mathbb{E}_{(X_t\sim\pi_t(x))_{t\leq s}} \left[\exp \left(-\int_0^s J(X_t) dt \right) \right] \approx \frac{1}{M}\sum_{i=1}^M \underbrace{\exp\left(-\sum_{t=0}^{s/\delta_s}\delta_s \cdot J(X_t^i)\right)}_{\exp(-W^{i,s/\delta_s})}
    \end{align*}
    where we obtain path $X_t$ from $X_{t-\delta_s}$ using a numerical discretization of a SPDE dynamics leaving $\pi_t(x)$ invariant through MH adjustment, with $X_0\sim\pi_0=\mathbb{P}^{\text{ref}}$ as initialization. Setting $s=1$, the average of a statistics $h(\cdot)$ over the path w.r.t the target $Q(x)$ can be computed as 
    \[ \frac{\sum_{j=1}^M h(X_n^j) e^{-W^{j,n}(X^j)} }{\sum_{j=1}^M e^{-W^{j,n}(X^j)} } \]
    on the final set of $M$ trajectories $\{X_n^j\}_{j=1}^M$, with $n=1/\delta_s$ as the number of switching steps.
\end{restatable}

In practice, we also have to discretize the $X_t$ path along the $[0,T]$ physical time, as shown in Algorithm \ref{alg:spde} in detail. Some drawbacks of this non-equilibrium approach: (1) simulating SPDE with MH in general is computationally intensive; (2) the weights may have high variance depending on how fast $\pi_s$ changes (compared to equilibrium methods, where each path has weight $\approx 1$ and the switching can be faster). %We illustrate these benefits in the experimental section. 
All proofs and additional details for this section can be found in Appendix~\ref{app:control_methods}.

\section{Reformulating Gibbs Variational Principle via Wasserstein Dynamics}
\label{sec:WGF} 
%     The following classical identity over $\mathcal{P}(\mathbb{R}^d)$
% \[-\log\mathbb{E}_{\text{prior} (\beta)} \mathcal{L}_{y,X}(\beta)=\min_{\rho\ll \mathbb{P}_{\text{prior}}(\beta)} \{-\mathbb{E}_\rho [\log\mathcal{L}_{y,X}(\beta)]+\text{KL}(\rho \Vert \mathbb{P}_{\text{prior}} (\beta))\}\]
% has an analogue in the path space:
% \[-\log \mathbb{E}[e^{-f(W)}]=\inf_v\; \mathbb{E}\left[\frac{1}{2}\int\|v_s\|^2 ds+f(W+\int_0 v_s ds)\right]\]
% for any function $f$ on path space. A commonly used example is when $f(\cdot)=-\log g(X_T^u)$ condition on $X_0$. Above the minimizer $\rho^*$ is precisely the posterior when $\mathcal{L}_{y,X}(\beta)$ is the likelihood function. 

There is a variational formulation for the posterior path measure (that can be obtained by generalizing from $\mathcal{P}(\mathbb{R}^d)$) so in our case we can rewrite the problem as 
    \begin{equation}
    \label{eqn:loss_path}    \arg\min_{\mathbb{Q}^\theta:\mathbb{Q}^\theta_0=\rho_0}\;\; \text{KL}(\mathbb{Q}^\theta\Vert Q^*)\quad \Leftrightarrow\quad \arg\min_{\mathbb{Q}^\theta:\mathbb{Q}^\theta_0=\rho_0}\;\; \mathbb{E}_{\mathbb{Q}^\theta}[J(x;y)]+\text{KL}(\mathbb{Q}^\theta\Vert\mathbb{P}^{\text{ref}})
    \end{equation}
    where $Q^*$ is the target from \eqref{eqn:Q_star} and the solution is unique under mild assumptions. Equivalently in terms of the drift, using Girsanov's theorem we have
    \begin{equation}
    \label{eqn:loss_b}
    \arg\min_{b^\theta}\;\; \int_0^T\mathbb{E}_{x_t^\theta\sim q_t^\theta}\left[  J_t(x_t^\theta; y)+\frac{1}{4}\|b_t^\theta(x_t^\theta)\|^2 \right] \, dt \, .
    \end{equation}
While the loss \eqref{eqn:loss_b} can in principle be estimated via empirical samples and used to optimize over $b^\theta$, there are a few challenges: (1) the loss is complicated as a function of the drift $b^\theta$ (even though in terms of $\mathbb{Q}^\theta$ in \eqref{eqn:loss_path} it's convex); (2) one needs to calculate $\partial x_t^\theta/\partial \theta$ generated with SDE \eqref{eqn:posterior_sde} for updating $b^\theta$, which requires back-propagating through the solver. This approach is related to some of the Neural SDE methods for fitting SDEs to data \cite{li2020scalable,opper2019variational}.

Taking a different route based on convex optimization, we first leverage the Benamou-Brenier \cite{benamou2000computational} reformulation to pose \eqref{eqn:loss_path} as joint learning over $(q_t,b_t)_t$, which yields insights that will allow us to simplify further later. For our specific case, it becomes
%\begin{subequations}
\begin{equation}
\begin{aligned}
\label{eqn:BB}
&\min_{q_t,b_t} \;\int_0^T \int_{\mathbb{R}^d} \left[\frac{1}{4}\|b_t(x_t)\|^2  + J_t(x_t; y)\right] q_t(x_t) \, dx dt 
&\text{s.t. } \;\partial_t  q_t = -\nabla\cdot(q_t (b_t+u_t^{\text{ref}}))+\Delta q_t, q_0=\rho_0 %, \rho_T^\theta=\delta_B
\end{aligned}
\end{equation}
%\end{subequations}

Now if one follows the augmented Lagrangian approach in \cite{benamou2000computational} that via reparameterization turns the problem into a convex optimization, the Laplacian $\Delta q_t$ from the diffusion in Fokker-Planck introduces a challenge as we have to deal with a Bilaplacian type of quantity in one of the updates. %They consider deterministic dynamics for $\mathcal{W}_2$ in OT.

\subsection{Approximation}
\label{sec:approximation}
For \eqref{eqn:BB}, there is a way to move the Laplacian term into the objective that is used in various works (e.g., Eqn (4.25) in \cite{chen2021stochastic}). Followed by the similar $(q_t,b_t)\mapsto (q_t,m_t)$ for $m_t=q_t\cdot b_t$ transformation as done in \cite{benamou2000computational}, we end up with our \emph{jointly convex} objective with an extra (relative) Fisher information in the objective but a continuity equation as the constraint.
%     \begin{align}
%     \label{eqn:fisher_objective}
% &\min_{q_t,m_t} \;\int_0^T \int_{\mathbb{R}^d} \left[\frac{1}{2}\left\|\frac{m_t}{q_t}(x_t)\right\|^2 +\frac{1}{2}\|\nabla \log q_t(x_t)\|^2 +J(x_t; y)\right] q_t(x_t) \, dx dt\\
% &\text{s.t. } \;\partial_t  q_t = -\nabla\cdot m_t, q_0=\rho_0 \nonumber 
% \end{align}

%\begin{proposition}
\begin{restatable}[Convex Reformulation]{proposition}{rewrite}
\label{lem:rewrite_BB}
For a general reference $u_t^{\text{ref}}\neq 0$, \eqref{eqn:BB} is equivalent to
\begin{align}
&\min_{q_t,m_t} \;\int_0^T \int_{\mathbb{R}^d} \left[\frac{1}{4}\left\|\frac{m_t}{q_t}(x_t)-(u_t^{\text{ref}}-\nabla \log\mathbb{P}_t^{\text{ref}})(x_t)\right\|^2 +\frac{1}{4}\left\|\nabla \log\frac{q_t}{\mathbb{P}_t^{\text{ref}}}(x_t)\right\|^2 +J_t(x_t)\right] q_t(x_t)  dx dt \nonumber\\
& \quad + \frac{1}{2}\int_{\mathbb{R}^d} q_T(x_T) \log \frac{q_T}{\mathbb{P}_T^{\text{ref}}}(x_T) \, dx_T \nonumber\\
&\text{s.t. } \;\partial_t  q_t = -\nabla\cdot m_t,\; q_0=\rho_0 \label{eqn:relative_fisher_obj}\, .%, \rho_T^\theta=\delta_B 
\end{align}
When the reference process $\mathbb{P}^{\text{ref}}$ is in equilibrium (i.e., $\mathbb{P}^{\text{ref}}_t=\nu\propto e^{-V}\; \forall t$) the problem reduces to
\begin{align}
&\min_{q_t,m_t} \;\int_0^T \int_{\mathbb{R}^d} \left[\frac{1}{4}\left\|\frac{m_t}{q_t}(x_t)\right\|^2 +\frac{1}{4}\left\|\nabla \log\frac{q_t}{\nu}(x_t)\right\|^2 +J_t(x_t)\right] q_t(x_t) \, dx dt+\frac{1}{2}\int_{\mathbb{R}^d} q_T(x_T) \log\frac{q_T}{\nu}(x_T)\, dx_T \nonumber\\
\label{eqn:simplified_objective}
&\text{s.t. }\; \partial_t  q_t = -\nabla\cdot m_t,\; q_0=\rho_0
\end{align}
In this case, $\rho_0=\nu=\mathbb{P}_0^{\text{ref}}$ for the initial distribution. 
%\end{proposition}
\end{restatable}

\paragraph{Reformulation in terms of marginal density:} Compared to the dynamic formulation of $\mathcal{W}_2$ optimal transport from \cite{benamou2000computational}, in Lemma~\ref{lem:rewrite_BB} we have an extra (relative) Fisher information and potential energy term involving $J$, and removal of the terminal $\rho_T$ constraint for a penalty on $q_T$. In what follows, we adopt the following approximation, which crucially only involves the evolution of a \emph{single} variable, instead of two. Since the objective \eqref{eqn:simplified_objective} is (jointly) convex and (almost) separable over time $t\in[0,T]$ we propose to solve a sequence of ``Wasserstein Gradient Flow"/JKO steps between adjacent marginals as (initialized at $q_0=\rho_0$)
\begin{equation}
\label{eqn:WGF}
q_{t+h}\leftarrow \arg\min_{q\in \mathcal{P}(\mathbb{R}^d)}\; \frac{1}{2h^2} \mathcal{W}_2^2(q,q_{t})+F(q)
\end{equation}
for advancing the solution over $t\in [0,T]$ with stepsize $h$, where $F(q)$ is the convex relative Fisher information functional $I(q;\nu):=\frac{1}{2}\int \|\nabla \log q(x)/\nu (x)\|^2 q(x) dx$ + term linear in $q: \frac{2}{h} \int J_{t+h}(x)q(x) dx$ . 
To see one can work with $\{q_t\}_t$ only, observe that in \eqref{eqn:simplified_objective}, the drift term $v_t$ (or equivalently $m_t$) only appears in the first part, and the rest is a function of the density $q_t$ only. Therefore given a curve of optimal measure $\{q_t\}_{t\leq T}$ one can solve classical OT between adjacent pairs to get the corresponding vector fields $\{v_t\}_t$ that follow the trajectory with \emph{minimum} kinetic energy (hence divergence-free). This is an Action Matching problem \cite{neklyudov2023action} if provided with samples from the $\{q_t\}_{t\leq T}$ curve.

Conceptually, this is a dynamical system perspective on the problem as a sequence of transformations
\[q_0\xrightarrow{b_h}q_h\rightarrow\cdots\rightarrow q_{T-h}\xrightarrow{b_T}q_T\,,\]
and we study the joint distribution over $\mathcal{P}(\mathbb{R}^d\times\mathbb{R}^d)$ at a time, as the Markovian representation \eqref{eqn:posterior_sde} suggests. This reduces the problem from a joint $T \times d$ distribution to a sequence of $2d$ problems. Unlike the more global method in Section \ref{sec:track_b_transport} that repeatedly updates the path measure $\mathcal{P}(\mathcal{C}([0, T], \mathbb{R}^d))$, it is a one-pass method evolving $\mathcal{P}(\mathbb{R}^d)$ from $t=0$ to $t=T$, local in nature. Alternative approximations on \eqref{eqn:simplified_objective} and more justification for our approach based on \eqref{eqn:WGF} are discussed in Section~\ref{sec:entropic_OT}.

We view the methods in the following sections as more principled than simply writing out the Lagrangian for \eqref{eqn:BB} and fitting $2$ NNs to solve a min-max problem over the $2$ variables across $t\in [0,T]$. This is the approach adopted in \cite{neklyudov2023computational} where gradient descent/ascent is done on the $\phi_t,q_t$ functions (solution of 2 coupled PDEs), subject to (multi-)marginal constraints: specializing to the case of SB it is (c.f. eqn (28) therein)
\begin{equation}
\label{eqn:AM}
\inf_{q_t\in\Pi(\rho_0,\rho_1)} \sup_{\phi_t}\quad \int\phi_1(x)\rho_1(x) dx-\int\phi_0(x)\rho_0(x) dx-\int_0^1\int \left(\frac{1}{2}\|\nabla\phi_t\|^2+\frac{1}{2}\Delta \phi_t+\partial_t \phi_t\right)q_t\, dx dt\, ,
\end{equation}
which can be viewed as a PINN regularizer for the dynamical dual formulation of SB:
\[\min_{Q_0=\rho_0,Q_1=\rho_1} \text{KL}(Q\Vert W) = \sup_{\phi_t} \; \int\phi_1(x)\rho_1(x) dx-\int\phi_0(x)\rho_0(x) dx\quad\text{s.t.}\; \frac{1}{2}\|\nabla\phi_t\|^2+\frac{1}{2}\Delta \phi_t+\partial_t \phi_t = 0\]
with collocation points drawn from $q_t$. Operationally, \eqref{eqn:AM} amounts to a series of Action Matching and WGF updates, with both $\phi,q$ parametrized as NNs. The re-writing, although involving only $1$ variable, becomes a PDE-constrained optimization that is generally difficult to solve.
% This re-writing suggests if one is only interested in identifying the vector field, the optimization / evaluation against $\rho_t$ is not strictly necessary (i.e., objective can be used off-policy). 
% Similar AM objective is proposed in \cite{albergo2024nets} where a regularizer on the Hamilton-Jacobi equation is used (and one gets OT solution instead of eOT). 

\subsection{Low Dimensional Case: Composite Optimization for \eqref{eqn:relative_fisher_obj}}
\label{sec:low_dim_eulerian}

For the most general case \eqref{eqn:relative_fisher_obj}, since the resulting optimization is jointly convex in the variables -- the system in \eqref{eqn:relative_fisher_obj} should be contrasted with Theorem 1 in \cite{du2024doob} -- one may discretize the space and apply convex optimization methods on the resulting objective.

The formulation in \eqref{eqn:relative_fisher_obj} requires $\nabla \log \mathbb{P}_t^{\text{ref}}$ for the reference process, which we typically don't know analytically. But assuming (1) one can sample from the prior trajectories; (2) $u_t^{\text{ref}}=-\nabla V_t$ takes gradient form, an Action Matching loss \cite[Algorithm 1]{neklyudov2023action} can be used to learn this, where we expect $\nabla s_t^*(x)\approx -\nabla V_t(x)-\nabla\log\mathbb{P}_t^{\text{ref}}(x)$. In the case of TPS with the Brownian bridge as reference, we have $\nabla\log\mathbb{P}_t^{\text{ref}}(x_t) = -\frac{x_t-A}{2t}$ 
analytically available. % this is something the min-max approach has to estimate as well
% Can one get away with estimating the marginal with conditional? 
\cite[Lemma B.1]{pidstrigach2025} gives another way to compute this score, which also only requires simulating from the reference process and fitting a least squares objective.
 % but requires the deterministic Jacobian factor computed as: \[J_{t|s} = J_{s|s}\cdot e^{\int_s^t \nabla u_r^{\text{ref}}(X_r)\, dr} = I_d\cdot e^{\int_s^t \nabla u_r^{\text{ref}}(X_r)\, dr}\] if the diffusion coefficient is position-independent

With \eqref{eqn:relative_fisher_obj} in mind, we can adapt the iterative procedure in \cite{carlier2023wasserstein} to solve this convex problem via a primal-dual method (a related Augmented Lagrangian method for WGF is also considered in \cite{benamou2016augmented}). Our method reduces to minimizing a finite-dimensional composite objective of type 
\[\min_y\; \Phi(y)+\Psi(\Lambda y)\]
for each time stepping from $t$ to $t+h$ -- therefore we solve in total $T/h$ problems of such type. Above $\Phi,\Psi$ are both convex functions that admit an easy-to-compute proximal operator. Departing from the min-max formulation \cite{neklyudov2023computational}, this is generally simpler, with convergence ensured. We give the details of this Eulerian solver in Appendix~\ref{app:eulerian}.

The method above, though NN-free, requires updates on a \emph{fixed grid} in $\mathbb{R}^d$ (since it uses an Eulerian representation of the discretized density over space), so may run into scalability issues in high dimension.

\subsection{High Dimensional Case: Lagrangian Methods for \eqref{eqn:simplified_objective}}
\label{sec:lagrangian_methods}

% A second class of approaches rather uses a Lagrangian representation, which is well adapted to optimal transport where the thought after solution is obtained by warping the density at the previous iterate. 

% this is a description on the marginals (we only had description on the joint before)  
% we have $Q(x), (b_t)_t, (\rho_t)_t$ 

In the high-dimensional regime, we will work with the simpler problem of \eqref{eqn:simplified_objective} that conditions an equilibrium process with observations, and advocate for parameterizing $q_t$ as either an interacting particle system, or as a parametrized push-forward map \cite{bunne2022proximal} that allows us to infer $b_t$ easily. In both cases, we leverage the perspective put forth in \eqref{eqn:WGF} so optimization is only done between adjacent marginals to track its evolution, and not joint learning over $q_t,b_t$. We will heavily rely on the representation of Fisher information as 
\begin{align}
R(q_t):=\frac{1}{2}\int \|\nabla \log q_t(x)\|^2 q_t(x) dx &= 2\int \|\nabla \sqrt{q_t(x)}\|^2 dx = \frac{1}{2}\int \frac{\|\nabla q_t(x)\|^2}{q_t(x)}dx \nonumber\\
&= \sup_{\frac{1}{2}\|\beta(x)\|^2+\alpha(x) \leq 0} \int \nabla q_t(x)^\top \beta(x) + q_t(x) \alpha(x) dx \label{eqn:fisher_rewrite_1}\\
&= \sup_{\frac{1}{2}\|\beta(x)\|^2+\alpha(x) \leq 0}\int (\alpha(x)-\nabla\cdot \beta(x))\cdot q_t(x) dx \label{eqn:fisher_rewrite_2}
\end{align}
and its approximation with empirical samples below. All proofs for this subsection are given in Appendix~\ref{app:sec_4}.
\begin{restatable}[FI Estimator]{lemma}{FIestimator}
\label{lem:FI_estimator}
Given $n$ samples $\{x_i\}_{i=1}^n\sim q_t$, one can efficiently estimate the Fisher Information functional as $\hat{R}(q_t)\approx \frac{1}{n}\sum_{i=1}^n \left(\frac{1}{m}\sum_{j=1}^m\frac{\|y_i^j-x_i\|^2}{2\sigma^4}\right)$ for $y_i^j\sim  \mathcal{N}(x_i,\sigma^2 I), j=1\dots, m$, and $\sigma$ a bandwidth parameter.
\end{restatable}

\subsubsection{Sequence of Pushforward Maps}
\label{sec:push_forward_map}

The first Lagrangian approach we consider is based on Brenier's Theorem for $\mathcal{W}_2$-OT, which says that the optimal transport map is given by the gradient of a convex function that solves a Monge-Amp\`ere equation. To leverage this push-forward map representation \cite{bunne2022proximal}, we rewrite \eqref{eqn:WGF} as  
\begin{equation}
\begin{aligned}
\label{eqn:jko_map}
\phi_t \leftarrow \arg\min_{\phi\, \text{convex}}\;\; \frac{1}{2h^2} \int \|\nabla \phi(x)-x\|_2^2 \cdot q_{t}(x) dx +F(\nabla \phi_\# q_{t}); \quad
q_{t+h} \leftarrow \nabla \phi_{t\#}q_{t}
\end{aligned}
\end{equation}
and propose to use an Input-Convex NN architecture \cite{amos2017input} to optimize over a family of convex functions $\phi^\theta$ that takes $t,\{x_t^i\}_{i=1}^n\sim q_t$ and output $\{x_{t+h}^i\}_{i=1}^n\sim q_{t+h}$. The gradient $\nabla \phi^\theta$ gives us the desired transport map. This allows us to evolve an empirical measure $\hat{q}$ in time and space. 

In order to evaluate $F(\nabla \phi_\# q_{t})$ for optimization, the "internal energy" term $\int \|\nabla\log q_t(x_t)\|^2 d q_t(x_t)$ that is nonlinear in $q_t$ seems challenging to estimate with empirical measure $\hat{q}_t$, and using the change of variables formula involving determinant along the dynamics is computationally intensive. However, here one can linearize via bi-level optimization using \eqref{eqn:fisher_rewrite_1} as we illustrate in Algorithm \ref{alg:jko} (the rest of the terms in $F$ are easily linearizable with possible integration-by-parts). Lemma~\ref{lem:FI_estimator} gives a $0$th-order method for approximating the ``gradient", and luckily here we only need to draw samples from $q_t$ for implementation, \emph{without} evaluating its density. The particular form \eqref{eqn:fisher_rewrite_2} is a special feature of Fisher Information -- other nonlinear energies may be difficult to re-parameterize to be linearizable.

% This property suggests that we only need to draw samples from ρt and need not evaluate its density. Nonlinear potential energies will require reparameterization to be linearizable.

%Unlike the Schrodinger Bridge problem, the Hopf-Cole transform does not linearize the dual objective in density. Thus, we cannot approximate the dual using only the Monte Carlo estimate.

\subsubsection{Interacting Particle System}  %  and PDE
\label{sec:IPS}

% As a way to understand the WGF scheme, note that with the potential energy $J_t$ term, iterating \eqref{eqn:WGF} (or equivalently \eqref{eqn:jko_map}) as $h\rightarrow 0$ gives the PDE \cite{gianazza2009wasserstein}

To understand the iterative scheme \eqref{eqn:WGF}, which is better termed acceleration in Wasserstein geometry with $I(q;\nu)$ functional, we see that although one description on the optimal marginal density evolution in $\mathcal{P}(\mathbb{R}^d)$ is through the pair of PDEs Fokker-Planck (evolved till time $T$): \[\partial_t  q_t = -\nabla\cdot(q_t (b_t+u^{\text{ref}}))+\Delta q_t,\, q_0=\rho_0\] and HJB on $b_t$ \eqref{eqn:path_integral_control}-\eqref{eqn:HJB} separately, it is more convenient to work with the following.

% \begin{equation}
% \label{eqn:Fisher_WGF}
% \partial_t q_t + \nabla\cdot\left(q_t\nabla\left(2\frac{\Delta \sqrt{q_t}}{\sqrt{q_t}}-2 J_t\right)\right) =\partial_t q_t + \nabla\cdot\left(q_t\nabla\left(\Delta \log q_t+\frac{1}{2}\|\nabla\log q_t\|^2-2 J_t\right)\right)= 0
% \end{equation}
% although \eqref{eqn:Fisher_WGF} is only an approximation.

\begin{restatable}[Optimality Condition]{proposition}{ER}
\label{prop:euler_lagrange}
The Euler–Lagrange equations for \eqref{eqn:simplified_objective} over $[0,T]$ take the more suggestive coupled PDE form (also implies the optimal $\nabla \theta_t$ takes gradient form)
\begin{align*}
&\partial_t  q_t = -\nabla\cdot(q_t \nabla\theta_t),\, q_0=\rho_0, \nabla\theta_0=b_0\\
&\partial_t \theta_t+\frac{1}{2}\|\nabla\theta_t\|^2 =-2(\Delta\sqrt{q_t})/\sqrt{q_t} + 2\Delta\sqrt{\nu}/\sqrt{\nu}+ 2 J_t\\
&2 J_T(X_T)+\log \frac{q_T}{\nu}(X_T)+\theta_T(X_T)=0\, .
\end{align*}
Therefore for $t\in(0,T),$ the solution to \eqref{eqn:simplified_objective} obeys
\begin{align}
\dot{X}_t &= \nabla\theta_t(X_t) \nonumber\\
\ddot{X}_t &= \nabla\left[\frac{\partial \theta_t(X_t)}{\partial t}+\frac{1}{2}\|\nabla\theta_t(X_t)\|^2\right]=-\left(2\nabla \frac{\Delta\sqrt{q_t}}{\sqrt{q_t}}-2\nabla \frac{\Delta\sqrt{\nu}}{\sqrt{\nu}}-2\nabla J_t\right)(X_t)\label{eqn:Fisher_WGF}\\
&=-\nabla\left(\Delta \log q_t+\frac{1}{2}\|\nabla\log q_t\|^2 +\Delta V-\frac{1}{2}\|\nabla V\|^2 -2J_t\right)(X_t) =:-\nabla F(q_t)(X_t)\, . \nonumber
\end{align}
This implies in \eqref{eqn:Fisher_WGF} we are accelerating with force field $F$ as in Newton's law $\ddot{X}_t =-\nabla F(q_t)(X_t).$
\end{restatable}

Proposition~\ref{prop:euler_lagrange} gives a proper meaning to acceleration on the particle level. Compared to OT that gives constant-speed geodesic $\ddot{X}_t=0$, here we have that acceleration is given by both the gradient of the spatial potential and an extra ``potential energy" term from the \text{KL} functional $\frac{1}{2}\|\nabla_{\mathcal{W}_2} \text{KL}(q\Vert \nu)\|_q^2=I(q;\nu)$. In the absence of $J$ and the terminal constraint, one expects $q_T\rightarrow  \nu$ asymptotically.  % V = entropy here, \nabla V = \nabla \log \rho
% Velocity and acceleration are close when the particle is accelerating gently and hasn’t been moving long (i.e., $T$ small). This is more likely at low speeds (i.e., with smoothly varying $q_t$). 
% However, using the PDE $\leftrightarrow$ probability flow ODE/SDE conversion, which preserves the marginal (but not the joint)

Written in integral form, the trajectory of the particle is therefore given by $(X_0\sim \rho_0, \nabla \theta_0(X_0)\approx 0)$
\begin{align}
X_h= X_0+h\nabla \theta_0(X_0)-&\int_0^h\int_0^s \nabla F(q_t)(X_t)\, dt ds=X_0+h\nabla \theta_0(X_0)-\int_0^h \nabla F(q_t)(X_t)(h-t)\, dt\label{eqn:x_exact}\, .%\\
%\nabla \theta_h (X_h) &= \nabla\theta_0(X_0)-\int_0^h \nabla F(q_t)(X_t) dt\, . 
\end{align}
To numerically simulate \eqref{eqn:Fisher_WGF} to approximately follow the optimal marginal density curve $q_t$, one can kernelize the gradient field in \eqref{eqn:Fisher_WGF} with an interacting particle system. The nonlinear term, $\nabla (\Delta\sqrt{q}/\sqrt{q})$, which corresponds to $\nabla \delta R(q)$ \cite{gianazza2009wasserstein}, at first glance, poses a challenge for estimation with samples from $q_t$ only, but similar linearization ideas as earlier can be used here for approximation. In particular, the relation \eqref{eqn:fisher_rewrite_2} shows that $\delta R(q_t)(x) = \alpha^*(x)-\nabla \cdot\beta^*(x)$ from Danskin's Theorem, where $\alpha^*(x)=-\frac{1}{2}\|\nabla \log q_t(x)\|^2, \beta^*(x)=\nabla\log q_t(x)$ -- this representation of the first variation can bypass density estimation and becomes amenable to particle implementation via $\mathcal{K}_{q_t}\nabla \delta R(q_t)(z) =\int \mathcal{K}(x-z)\nabla [\alpha^*(x)-\nabla \cdot\beta^*(x)]q_t(x) dx$ that can be used in $\nabla F(q_t)$ for running the $2$nd-order ODE dynamics \eqref{eqn:x_exact}. More details can be found in Appendix~\ref{app:IPS}. Compared to Algorithm~\ref{alg:jko} from Section~\ref{sec:push_forward_map}, the benefit of this approach is that we don't need to fit a NN since everything is analytical, although we expect the performance to hinge on that of the kernelized gradient field approximation.

% the relative Fisher information FI is the Bregman divergence of the Fisher information J

\subsubsection{Entropic-regularized OT and Alternative Approximation}
\label{sec:entropic_OT}

To better interpret the methodology proposed in this section within a broader context, re-writing \eqref{eqn:simplified_objective}, between $t\in [0,h]$ our method solves (assume $J$ consists of $\delta$-functions as in \eqref{eqn:J_functional})
\begin{align}
&\min_{q_{t\in[0,h]}, v_t}\,\, \frac{1}{2h} \left(h\int_0^h \int_{\mathbb{R}^d}  q_t(x)\|v_t(x)\|^2 dxdt \right)+\int_0^h\int  \left[\frac{1}{2}\|\nabla \log \frac{q_t}{\nu}(x_t)\|^2+2 J_t(x_t)\right] q_t(x_t) dxdt \nonumber\\  % + \text{KL}(q_T\Vert\nu)
&\approx \frac{1}{2h} \left(h\int_0^h \int_{\mathbb{R}^d}  q_t(x)\|v_t(x)\|^2 dxdt \right)+h\cdot \int  \frac{1}{2}\|\nabla \log \frac{q_h}{\nu}(x_h)\|^2 q_h(x_h) dx +2\int J_h(x_h) q_h(x_h) dx
\label{eqn:approx_explicit}
\end{align}
subject to continuity equation and $q_0=\rho_0$. In \eqref{eqn:WGF} we used the 2nd line as approximation (after rescaling). But a different grouping on the 1st line leads to (using argument very similar to Proposition~\ref{lem:rewrite_BB}, also see \cite{conforti2021formula})
\begin{equation}
\label{eqn:entropic_W2}
= \min_{q_h}\; \Big\{\underbrace{\min_{\pi\in \Pi(\rho_0,q_h)}\{ \text{KL}(\pi\Vert R_{0,h})\}}_{\zeta_h(\rho_0,q_h)} +\text{KL}(\rho_0\Vert\nu)-\text{KL}(q_h\Vert\nu)+2\int J_h(x_h) q_h(x_h) dx\Big\}
\end{equation}
where $R_{0,h}$ is the joint law at time $0,h$ of the reference SDE with invariant measure $\nu$. From \cite[Eqn (1.10)]{conforti2021formula} we have that in the short-time limit  
\begin{equation}
\label{eqn:approximation_cite}
h\cdot\zeta_h(\rho_0,q_h)= \underbrace{\frac{1}{2}\mathcal{W}_2^2(\rho_0,q_h)+h(\text{KL}(q_h\Vert\nu)-\text{KL}(\rho_0\Vert\nu))}_{(\star)}+\underbrace{\frac{h^2}{2}\int \int_0^1 \|\nabla \log \frac{q_t^0}{\nu}\|^2q_t^0 dt dx}_{\approx h^2\cdot I(q_h;\nu)} + o(h^2)
\end{equation}
for $(q_t^0)_t$ the (unique) Wasserstein geodesic between $\rho_0$ and $q_h$. Dividing by $h$, this recovers \eqref{eqn:approx_explicit} and justifies the derivation in \eqref{eqn:WGF}, giving a quantitative bound on the approximation. 

It is clear that with formulation \eqref{eqn:entropic_W2}, we only need to solve at each new observation time $\{t_k\}_{k=1}^K$ -- since without the $J$ term, its solution is given by $q_h=\nu$ if $\rho_0=\nu$. For reversible Brownian motion $W$ as reference (i.e., $\nu=\text{Leb}$), \eqref{eqn:entropic_W2} amounts to solving between $2$ observations (after cancellation):
\begin{equation}
\label{eqn:eOT_GF}
q_{t_{i+1}} \leftarrow \arg\min_{q}\;  \mathcal{W}_{2,t_{i+1}-t_i}^2(q,q_{t_i})+\bar{F}(q)
\end{equation}
with entropy-regularized $\mathcal{W}_2$ distance
\[\mathcal{W}_{2,t_{i+1}-t_i}^2(q,q_{t_i}) = \min_{\pi^i \in \Pi(q_{t_i}, q)}\int \frac{1}{2(t_{i+1}-t_i)}\|x-y\|^2d\pi^i(x,y)+\text{KL}(\pi^i\Vert q_{t_i}\otimes q)\]
as the distance-inducing metric and the likelihood $2\int J(x) dq(x)$ term as the objective functional $\bar{F}$. We give methods for implementation with this objective, along with additional discussions in Appendix~\ref{app:entropic_ot}. While \eqref{eqn:eOT_GF} has a simpler objective than \eqref{eqn:WGF}, it does not admit a push-forward map solution \eqref{eqn:jko_map} or modeling by ODE dynamics \eqref{eqn:Fisher_WGF}. Nevertheless, the sequential reduction implies that we still track the $\{q_t\}_t$ dynamics as a consecutive optimization over $\mathcal{P}(\mathbb{R}^d\times\mathbb{R}^d)$, instead of a multi-marginal problem on the joint space $\mathcal{P}(\mathbb{R}^d)^K$.

\begin{remark}  % to $\min_{q_0=\rho_0,q_{t_1}=\rho}\, \text{KL}(q_{0,t_1}\Vert R_{0,t_1})$
\label{rmk:entropic_OT}
The more serious concern for this regularized-OT approach \eqref{eqn:eOT_GF}, compared to Algorithm~\ref{alg:jko} and Proposition~\ref{prop:euler_lagrange} that only require $\nabla V$, is its generalization to $u^{\text{ref}}\neq 0$. Importantly, we show both $\mathcal{W}_2^2$ and $I(q;\nu)$ are easily estimable in Algorithm~\ref{alg:jko} with empirical samples. In \eqref{eqn:eOT_GF}, although we can time step across $t\in [0,T]$ with a larger interval $t_{i+1}-t_i$ and it's expected to be more accurate, the inner optimization $\mathcal{W}_{2,t_{i+1}-t_i}^2$ is not easily solvable without knowledge of the analytical transition kernel of the reference process. And the only asymptotic one may resort to is $R_{t_i,t_{i+1}}\rightarrow \nu\otimes \nu$ as $t_{i+1}-t_i\rightarrow \infty$, which is too loose for most practical purpose. In contrast, in \eqref{eqn:WGF}, we pay the price of having to do smaller step-size $h$ between adjacent marginals because of the $o(h^2)$ approximation resulting from evaluating $I(q_t;\nu)$ at a single marginal instead of along the entropic interpolation, but if $h$ is small, even the first term $(\star)$ in \eqref{eqn:approximation_cite} involving $\frac{1}{2}\mathcal{W}_2^2$ furnishes a $o(h)$ approximation \cite{conforti2021formula}.  % to $h\cdot\zeta_h(\rho_0,\rho)$
\end{remark}

% \section{Experiments}
% We include some numerical demonstrations below, with more in Appendix~\ref{sec:numerics}.

\section{Conclusion}
In this work, we proposed two classes of new methodologies for bringing NN training to the problem of posterior path sampling. The first is based on a \emph{controlled} dynamics over path space in order to progressively learn the optimal $b_t^*$ vector field in a memory-efficient way (Sec \ref{sec:track_b_transport}), and is guaranteed to hit target $Q$ at \emph{finite} time without additional reweighting, or solving high-dimensional PDEs. The second is based on convex optimization in Wasserstein space and OT, where we gradually evolve a curve of marginal densities over time. We propose grid-free stochastic methods, based either on push-forward maps (Sec \ref{sec:push_forward_map}), or an interacting particle system (Sec \ref{sec:IPS}), that can avoid the curse of dimensionality, while maintaining generality and good approximation properties (Sec \ref{sec:entropic_OT}).

% \section*{References}
% References follow the acknowledgments in the camera-ready paper. Use unnumbered first-level heading for
% the references. 

\bibliography{main}

\begin{thebibliography}{10}

\bibitem{albergo2024nets}
Michael~S Albergo and Eric Vanden-Eijnden.
\newblock {NETS: A Non-Equilibrium Transport Sampler}.
\newblock {\em arXiv preprint arXiv:2410.02711}, 2024.

\bibitem{amos2017input}
Brandon Amos, Lei Xu, and J~Zico Kolter.
\newblock {Input convex neural networks}.
\newblock In {\em International conference on machine learning}, pages
  146--155. PMLR, 2017.

\bibitem{apte2007sampling}
Amit Apte, Martin Hairer, AM~Stuart, and Jochen Voss.
\newblock {Sampling the posterior: An approach to non-Gaussian data
  assimilation}.
\newblock {\em Physica D: Nonlinear Phenomena}, 230(1-2):50--64, 2007.

\bibitem{benamou2000computational}
Jean-David Benamou and Yann Brenier.
\newblock {A computational fluid mechanics solution to the Monge-Kantorovich
  mass transfer problem}.
\newblock {\em Numerische Mathematik}, 84(3):375--393, 2000.

\bibitem{benamou2019entropy}
Jean-David Benamou, Guillaume Carlier, Simone Di~Marino, and Luca Nenna.
\newblock {An entropy minimization approach to second-order variational
  mean-field games}.
\newblock {\em Mathematical Models and Methods in Applied Sciences},
  29(08):1553--1583, 2019.

\bibitem{benamou2016augmented}
Jean-David Benamou, Guillaume Carlier, and Maxime Laborde.
\newblock {An augmented Lagrangian approach to Wasserstein gradient flows and
  applications}.
\newblock {\em ESAIM: Proceedings and surveys}, 54:1--17, 2016.

\bibitem{beskos2008mcmc}
Alexandros Beskos, Gareth Roberts, Andrew Stuart, and Jochen Voss.
\newblock {MCMC methods for diffusion bridges}.
\newblock {\em Stochastics and Dynamics}, 8(03):319--350, 2008.

\bibitem{bolhuis2021transition}
Peter~G Bolhuis and David~WH Swenson.
\newblock {Transition path sampling as Markov chain Monte Carlo of
  trajectories: Recent algorithms, software, applications, and future outlook}.
\newblock {\em Advanced Theory and Simulations}, 4(4):2000237, 2021.

\bibitem{bunne2022proximal}
Charlotte Bunne, Laetitia Papaxanthos, Andreas Krause, and Marco Cuturi.
\newblock {Proximal optimal transport modeling of population dynamics}.
\newblock In {\em International Conference on Artificial Intelligence and
  Statistics}, pages 6511--6528. PMLR, 2022.

\bibitem{carlier2023wasserstein}
Guillaume Carlier, Jean-David Benamou, and Daniel Matthes.
\newblock {Wasserstein gradient flow of the Fisher information from a
  non-smooth convex minimization viewpoint}.
\newblock {\em Journal of Convex Analysis}, 2024.

\bibitem{carrillo2015numerical}
Jos{\'e}~Antonio Carrillo, Yanghong Huang, Francesco~Saverio Patacchini, and
  Gershon Wolansky.
\newblock {Numerical study of a particle method for gradient flows}.
\newblock {\em arXiv preprint arXiv:1512.03029}, 2015.

\bibitem{chen2021stochastic}
Yongxin Chen, Tryphon~T Georgiou, and Michele Pavon.
\newblock {Stochastic control liaisons: Richard sinkhorn meets gaspard monge on
  a schrodinger bridge}.
\newblock {\em Siam Review}, 63(2):249--313, 2021.

\bibitem{chizat2022trajectory}
L{\'e}na{\"\i}c Chizat, Stephen Zhang, Matthieu Heitz, and Geoffrey
  Schiebinger.
\newblock {Trajectory inference via mean-field langevin in path space}.
\newblock {\em Advances in Neural Information Processing Systems},
  35:16731--16742, 2022.

\bibitem{conforti2021formula}
Giovanni Conforti and Luca Tamanini.
\newblock {A formula for the time derivative of the entropic cost and
  applications}.
\newblock {\em Journal of Functional Analysis}, 280(11):108964, 2021.

\bibitem{du2024doob}
Yuanqi Du, Michael Plainer, Rob Brekelmans, Chenru Duan, Frank Noe, Carla~P
  Gomes, Alan Apsuru-Guzik, and Kirill Neklyudov.
\newblock {Doob's Lagrangian: A Sample-Efficient Variational Approach to
  Transition Path Sampling}.
\newblock {\em arXiv preprint arXiv:2410.07974}, 2024.

\bibitem{gentil2017analogy}
Ivan Gentil, Christian L{\'e}onard, and Luigia Ripani.
\newblock {About the analogy between optimal transport and minimal entropy}.
\newblock {\em Annales de la Facult{\'e} des sciences de Toulouse:
  Math{\'e}matiques}, 26(3):569--600, 2017.

\bibitem{gianazza2009wasserstein}
Ugo Gianazza, Giuseppe Savar{\'e}, and Giuseppe Toscani.
\newblock {The Wasserstein gradient flow of the Fisher information and the
  quantum drift-diffusion equation}.
\newblock {\em Archive for rational mechanics and analysis}, 194(1):133--220,
  2009.

\bibitem{heng2021simulating}
Jeremy Heng, Valentin De~Bortoli, Arnaud Doucet, and James Thornton.
\newblock {Simulating diffusion bridges with score matching}.
\newblock {\em arXiv preprint arXiv:2111.07243}, 2021.

\bibitem{holdijk2024stochastic}
Lars Holdijk, Yuanqi Du, Ferry Hooft, Priyank Jaini, Berend Ensing, and Max
  Welling.
\newblock {Stochastic optimal control for collective variable free sampling of
  molecular transition paths}.
\newblock {\em Advances in Neural Information Processing Systems}, 36, 2024.

\bibitem{li2020scalable}
Xuechen Li, Ting-Kam~Leonard Wong, Ricky~TQ Chen, and David Duvenaud.
\newblock {Scalable gradients for stochastic differential equations}.
\newblock In {\em International Conference on Artificial Intelligence and
  Statistics}, pages 3870--3882. PMLR, 2020.

\bibitem{liu2017stein}
Qiang Liu.
\newblock {Stein variational gradient descent as gradient flow}.
\newblock {\em Advances in neural information processing systems}, 30, 2017.

\bibitem{neklyudov2023action}
Kirill Neklyudov, Rob Brekelmans, Daniel Severo, and Alireza Makhzani.
\newblock {Action matching: Learning stochastic dynamics from samples}.
\newblock In {\em International conference on machine learning}, pages
  25858--25889. PMLR, 2023.

\bibitem{neklyudov2023computational}
Kirill Neklyudov, Rob Brekelmans, Alexander Tong, Lazar Atanackovic, Qiang Liu,
  and Alireza Makhzani.
\newblock {A computational framework for solving Wasserstein Lagrangian flows}.
\newblock {\em arXiv preprint arXiv:2310.10649}, 2023.

\bibitem{nusken2024stein}
Nikolas N{\"u}sken.
\newblock {Stein transport for Bayesian inference}.
\newblock {\em arXiv preprint arXiv:2409.01464}, 2024.

\bibitem{opper2019variational}
Manfred Opper.
\newblock {Variational inference for stochastic differential equations}.
\newblock {\em Annalen der Physik}, 531(3):1800233, 2019.

\bibitem{peluchetti2023diffusion}
Stefano Peluchetti.
\newblock {Diffusion bridge mixture transports, Schr{\"o}dinger bridge problems
  and generative modeling}.
\newblock {\em Journal of Machine Learning Research}, 24(374):1--51, 2023.

\bibitem{peyre2015entropic}
Gabriel Peyr{\'e}.
\newblock {Entropic approximation of Wasserstein gradient flows}.
\newblock {\em SIAM Journal on Imaging Sciences}, 8(4):2323--2351, 2015.

\bibitem{pidstrigach2025}
Jakiw Pidstrigach, Elizabeth Baker, Carles Domingo-Enrich, George
  Deligiannidis, and Nikolas Nüsken.
\newblock {Conditioning Diffusions Using Malliavin Calculus}.
\newblock {\em arXiv preprint arXiv:2504.03461}, 2025.

\bibitem{raginsky2024variational}
Maxim Raginsky.
\newblock {A variational approach to sampling in diffusion processes}.
\newblock {\em arXiv preprint arXiv:2405.00126}, 2024.

\bibitem{stoltz2007path}
Gabriel Stoltz.
\newblock {Path sampling with stochastic dynamics: Some new algorithms}.
\newblock {\em Journal of Computational Physics}, 225(1):491--508, 2007.

\bibitem{stoltz2010free}
Gabriel Stoltz, Mathias Rousset, et~al.
\newblock {\em {Free energy computations: A mathematical perspective}}.
\newblock World Scientific, 2010.

\bibitem{stuart2004conditional}
Andrew~M Stuart, Jochen Voss, and Petter Wilberg.
\newblock {Conditional path sampling of SDEs and the Langevin MCMC method}.
\newblock {\em Communications in Mathematical Sciences}, 2004.

\bibitem{sutter2016variational}
Tobias Sutter, Arnab Ganguly, and Heinz Koeppl.
\newblock {A variational approach to path estimation and parameter inference of
  hidden diffusion processes}.
\newblock {\em Journal of Machine Learning Research}, 17(190):1--37, 2016.

\bibitem{tong2023improving}
Alexander Tong, Kilian Fatras, Nikolay Malkin, Guillaume Huguet, Yanlei Zhang,
  Jarrid Rector-Brooks, Guy Wolf, and Yoshua Bengio.
\newblock {Improving and generalizing flow-based generative models with
  minibatch optimal transport}.
\newblock {\em arXiv preprint arXiv:2302.00482}, 2023.

\bibitem{triplett2025diffusion}
Luke Triplett and Jianfeng Lu.
\newblock {Diffusion methods for generating transition paths}.
\newblock {\em Journal of Computational Physics}, 522:113590, 2025.

\end{thebibliography}
\bibliographystyle{plain}

%%%%%%%%%%%%%%%%%%%%%%%%%%%%%%%%%%%%%%%%%%%%%%%%%%%%%%%%%%%%
\newpage
\appendix

%\section{Technical Appendices and Supplementary Material}
% Technical appendices with additional results, figures, graphs and proofs may be submitted with the paper submission before the full submission deadline, or as a separate PDF in the ZIP file below before the supplementary material deadline. There is no page limit for the technical appendices.

\section{Additional Details for Section~\ref{sec:controlled_path_measure}}
\label{app:control_methods}

We begin with a remark on the choice of the annealing path.
\begin{remark}
In the special case of TPS with $X_0=A, X_T=B$, there's an alternative path one can anneal against: The measure of interest $Q$ has a density w.r.t the Brownian bridge with the likelihood ratio being
\begin{equation}
\label{eqn:TPS_anneal}
J(x)\propto \frac{1}{2}\int_0^T \left(\frac{1}{2}\|\nabla V(X_t)\|^2-\Delta V(X_t)\right) dt=:\int_0^T J_t(X_t)\, dt
\end{equation}
up to a normalizing constant. Assuming one can sample from the Brownian bridge easily, this gives a way to interpolate between the Brownian bridge 
\begin{equation}
\label{eqn:BB_SDE}
dX_t=\frac{B-X_t}{T-t}dt+\sqrt{2}dW_t,\; X_0=A
\end{equation}
with
\begin{equation}
\label{eqn:prior_bb}
\pi_0(x)\propto \exp\left(-\frac{1}{4}\int_0^T \left\|\frac{dX_t}{dt}\right\|^2 dt\right)\delta_A(X_0)\delta_B(X_T)
\end{equation}
and a more general bridge with $-\nabla V_t$ as the potential. One may again intersperse MCMC with the transport step. With this annealing, we gradually introduce the reference potential instead of the constraint, i.e., we reweight the prior paths not based on $X_T=B$, but the potential $V$ instead. To see that \eqref{eqn:TPS_anneal}-\eqref{eqn:prior_bb} give the same path probability as the OM functional \eqref{eqn:Q_star}, we simply note that the extra factor of 
\[\int_0^T \frac{1}{2}\frac{dX_t}{dt}^\top u_t^{\text{ref}}(X_t) dt = -\int_0^T \frac{1}{2}\frac{dX_t}{dt}^\top \nabla V(X_t) dt=-\frac{1}{2}\int_0^T \frac{d V(X_t)}{dt} dt=\frac{1}{2}(V(X_0)-V(X_T))\]
is equal to $\frac{1}{2}(V(A)-V(B))$ (i.e., a constant fixed by the boundary condition).
\end{remark}

We give the two missing proofs from the controlled path measure Section~\ref{sec:controlled_path_measure} below.

\KRR*

\begin{proof}
Fix a particular $s$. Using the fact that the operator $\mathcal{S}$ is linear in $v$ if we view it as a function $v_t(x)$ indexed by $t$, we derive the solution to KRR for each $t$, following section A.1.2 from \cite{nusken2024stein}. The adjoint satisfies 
\[ \langle c,\mathcal{S}_t(v_t)\rangle_N = \langle\mathcal{S}_t^*c, v_t  \rangle_{\mathcal{H}_k} \quad \forall v_t \in\mathcal{H}_k^d,c\in\mathbb{R}^N\]
implying for $c\in\mathbb{R}^N$, using the definition of $h_s$,
\[\mathcal{S}_t^*c = \frac{1}{N}\sum_{i=1}^N \left(\frac{1}{2}\nabla_{x_t^i} k(\cdot, x_t^i)-\frac{1}{2}k(\cdot,x_t^i)\left(b_t(x_t^i)-\frac{dx_t^i}{dt}\right)\right) c_i\, .\]
for any $c\in \mathbb{R}^N$. Therefore the solution (parametrized by $N,\lambda$) is given by
\begin{equation}
\label{eqn:optimal_v}
v_t^* = \mathcal{S}_t^*(\lambda I_{N\times N}+\mathcal{S}_t\mathcal{S}_t^*)^{-1} j_t=: \mathcal{S}_t^*c\, ,
\end{equation}
which takes $\{x_t^i\}_{i=1}^N \in\mathbb{R}^d$ as input and output a $\mathbb{R}^d$-valued vector. Above $j_t[i] =-J_t(x_t^i)+\frac{1}{N}\sum_{j=1}^N J_t(x_t^j), i\in [N]$. Here the linear system to be solved in \eqref{eqn:optimal_v} involving $\mathcal{S}_t\mathcal{S}_t^*$ that maps $\mathbb{R}^d\times \mathbb{R}^d$ to $\mathbb{R}$ can be evaluated to be
\begin{align*}
(\mathcal{S}_t\mathcal{S}_t^*)_{ij}&= -\frac{1}{4}\left(b_t(x_t^i)-\frac{dx_t^i}{dt}\right)^\top\nabla_{x_t^j} k(x_t^i, x_t^j)+\frac{1}{4}\left(b_t(x_t^i)-\frac{dx_t^i}{dt}\right)^\top k(x_t^i,x_t^j)\left(b_t(x_t^j)-\frac{dx_t^j}{dt}\right)\\
&+\frac{1}{4}\nabla_{x_t^i}\cdot \nabla_{x_t^j} k(x_t^i, x_t^j)-\frac{1}{4}\nabla_{x_t^i}k(x_t^i,x_t^j)^\top\left(b_t(x_t^j)-\frac{dx_t^j}{dt}\right),\, i \in [N], j \in [N]
\end{align*}
for the corresponding kernel $k(\cdot,\cdot)$. \eqref{eqn:optimal_v} is a linear system of size $N\times N$ given $N$ empirical samples/trajectories $\{x_t^i\}_{i=1}^N$ at each $t\in [T]$. The resulting $(v_t^*)_t$ from above can be used to update the drift $(b_t)_t$ and generate new trajectories for solving \eqref{eqn:KRR} again at a new $s$, with possibly MCMC steps in between.
\end{proof}

\Jarzynski*
\begin{proof} 
%\[w_j=\frac{e^{-W^{j,n}}}{\sum_{i=1}^M e^{-W^{i,n}}}\]
Fix a switching with $n:=s/\delta_s$ steps, the proof 
follows closely the derivation in \cite[Remark 4.5]{stoltz2010free}. The MH adjustment on top of the SPDE dynamics makes sure the estimator is \emph{unbiased}, i.e.,
\[\frac{Z_s}{Z_0} = \mathbb{E}\left[\exp\left(-\sum_{t=0}^{s/\delta_s}\delta_s \cdot J(X_t)\right)\right]\]
even with the $\delta_s$ algorithmic time discretization.  
To see this, note the probability to observe a sequence of path $X_0,X_1,\dots, X_n$ under the dynamics is
\[\Pi(dX)=\pi_0(dX_0)\prod_{i=1}^n P_{\pi_i}(X_{i-1},dX_i)\]
where the transition probability $P_{\pi_i}(X_{i-1},dX_i)$ leaves the path measure $\pi_i(x)\propto \exp(-I(x)-i\cdot\delta_s J(x)) =: \exp(-V_i(x))$ invariant. 
Therefore
\begin{align*}
&\mathbb{E}_{\{X_i\}_{i=0}^n}\left[\exp\left(-\sum_{i=0}^{n}\delta_s \cdot J(X_i)\right)\right]=\int \prod_{i=0}^{n-1}\exp(-V_{i+1}(X_i)+V_i(X_i))\, \Pi(dX)\\
&= \frac{1}{Z_0}\int \exp(-V_1(X_0))\prod_{i=1}^{n-1} \exp(-V_{i+1}(X_i)+V_i(X_i))\,\prod_{i=1}^{n} P_{\pi_i}(X_{i-1},dX_i)dX_0\\
&= \frac{Z_1}{Z_0} \int \prod_{i=1}^{n-1} \exp(-V_{i+1}(X_i)+V_i(X_i))\,\pi_1(dX_1) \prod_{i=2}^n P_{\pi_i}(X_{i-1},dX_i)
\end{align*}
since by the invariance of kernel $P_{\pi_i}$ on $\pi_i$:
\[\int \exp(-V_i(X_{i-1}))P_{\pi_i}(X_{i-1}, dX_i) dX_{i-1} = Z_{i}\pi_i(dX_i)\, .\]
Inductively using the same argument gives 
\[\mathbb{E}_{\{X_i\}_{i=0}^n}\Bigg[\underbrace{\exp\left(-\sum_{i=0}^{n}\delta_s \cdot J(X_i)\right)}_{e^{-W^n}}\Bigg]=\frac{Z_n}{Z_0}\, ,\]
for which we approximate by averaging $M$ empirical samples in practice. To report the average of a statistics $h$ over the path, using importance sampling we have
\[\frac{1}{Z_n}\int h(x) \pi_n(x) dx = \frac{\int h(X) \frac{\pi_n(X)}{\rho_n(X)} d\rho_n(X)}{\int \frac{\pi_n(X)}{\rho_n(X)}d\rho_n(X)} \approx \frac{\sum_{j=1}^M h(X_n^j) e^{-W^{j,n}(X^j)} }{\sum_{j=1}^M e^{-W^{j,n}(X^j)} }\]
with $M$ independent walkers / trajectories. At final step $n$ of the algorithm, each trajectory $j\in[M]$ has individual weights as 
\[w_j=e^{-W^{j,n}}\]
where $W$ denotes the work along the MH-adjusted SPDE dynamics.
\end{proof}

\paragraph{Alternative Equilibrium Schemes}

For \emph{exactly} following the $\pi_s$ curve, one may be tempted to learn the drift term in the SPDE \eqref{eqn:langevin-spde} as a replacement for $\delta_x \log Q(x)$, in order to enforce this requirement. This will give a method similar to \cite{nusken2024stein}, but here with $v_s$ mapping $\mathcal{C}([0,T],\mathbb{R}^d)$ to $\mathcal{C}([0,T],\mathbb{R}^d)$:   % Fokker-Planck for SPDE exists under assumptions on the drift. 
    \begin{equation}
    \label{eqn:stein}
    ``\nabla \cdot(\pi_s(x) v_s(x))\overset{!}{=}\pi_s(x)\left(J(x)-\mathbb{E}_{\pi_s(x)}[J(x)]\right)=\partial_s \pi_s"
    \end{equation}% this is a continuity equation
    or equivalently written as the Stein equation (which is oblivious to normalizing constant)
    \begin{equation}
    \label{eqn:stein_rewritten}
``\nabla \log\pi_s(x)\cdot v_s(x)+\nabla\cdot v_s(x) \overset{!}{=} J(x)-\int J(x) d\pi_s"
\end{equation}
and update trajectories as 
    \[dx_s/ds = v_s(x_s),\; x_0\sim \pi_0.\] 
This method will directly move the trajectories instead of updating the drifts $b^s$ (as we do in Section~\ref{sec:track_b_transport}), and is closer to an interacting particle system, if we heuristically work with an implementation on \eqref{eqn:stein_rewritten} as (with each $t\in[T]$):
\begin{equation}
\label{eqn:stein_rewrite}
\arg\min_{v_s} \; \sum_{i=1}^N \left[\delta_x \log\pi_s(x^i)[t] \cdot v_s^t(x_t^i)+\nabla\cdot v_s^t(x_t^i) -\left(  J_t(x_t^i)  -\frac{1}{N}\sum_{n=1}^N  J_t(x_t^n)\right)\right]^2\,  \;\text{ for }\; x^i\sim \pi_s
\end{equation}
\[x_t^i \leftarrow x_t^i+\delta_s\cdot v_s^t(x_t^i)\quad\quad \forall i\in[N]\]
that moves trajectories so that the updated $\{x_t^i\}_{t\leq T}^{i=1,\dots, N}\sim \pi_{s+\delta_s}$. Note that the $T$ problems are coupled through terms like $d^2 x_t^i/dt^2$ that appear in \eqref{eqn:stein_rewrite} through the $1$st term. However, solving \eqref{eqn:stein_rewrite} means we'd need to keep $N\times \mathbb{R}^d\times T$ in memory, compared to our NN approach in Section~\ref{sec:track_b_transport} that keeps $\#\text{NN parameters}$ in memory. If one considers Fisher-Rao/Birth-Death on the trajectories with \eqref{eqn:stein} (which attach weights to the trajectories based on $J$), the memory footprint is similar. In Algorithm \ref{alg:transport_b} we pay for compute in terms of re-generating trajectories after each annealing step, but recover the $b_t^*$ drift prescribing the behavior of the dynamics as output. 

On the other hand, using \cite{albergo2024nets} with a PINN objective (for each $t\in[T]$) imposing \eqref{eqn:stein_rewritten} does not need interacting particle system, but will require re-simulating paths with the updated $v_s$ at each step and solve a \emph{two-parameter} loss evaluated on the trajectories. We remark that Algorithm 1 from \cite{albergo2024nets} also leverages explicit ``annealing" w.r.t the changing potential (assumed known), in addition to the learned transport $v_s$ part.  % $-\nabla V_t$ 

% \begin{remark}
% If the control $\alpha_t^*$ solves
% \begin{equation}
% \label{eqn:ais_pde}
% \nabla\cdot \alpha_t^*-\nabla V_t^\top\alpha_t^*-\partial_t V_t+\rho_t[-\partial_t V_t]=0
% \end{equation}
% then the SDE given as
% \[dX_t = \alpha_t(X_t)-\nabla V_t(X_t) dt+\sqrt{2}dW_t\]
% has time marginal as $\rho_t\propto e^{-V_t}$. Moreover, the deterministic ODE $\dot{X_t}=\alpha_t^*(X_t)$ with initial $X_0\sim \rho_0$ will follow $X_t \sim \rho_t$ for all $t$. This is the basis of PINN and AIS in \cite{arbel2021annealed}. Solution to this PDE \eqref{eqn:ais_pde} is not unique, however. 
% \end{remark}

\section{General Reduction and Eulerian Solver}
\label{app:eulerian}

We begin with a remark on the general approach adopted in Section~\ref{sec:WGF}, followed by the proof of our jointly convex reformulation involving $q_t,m_t$ from Section~\ref{sec:approximation}.

\begin{remark}
There are some degrees of freedom in the objective \eqref{eqn:loss_path}/\eqref{eqn:loss_b} by picking different decomposition of $Q^*$, e.g., using Brownian Bridge as $\mathbb{P}^{\text{ref}}$ with $J$ from \eqref{eqn:TPS_anneal} and $u_t^{\text{ref}}(x_t)$ from \eqref{eqn:BB_SDE}, in addition to the choice of $\mathbb{P}^{\text{ref}}$ as a prior, and $J$ the likelihood given by observations as in \eqref{eqn:J_functional}.  % =(B-x_t)/(T-t), still \sqrt{2} as diffusion coefficient
\end{remark}

\rewrite*
\begin{proof}
We give an argument here that is slightly different from Eqn (4.31) in \cite{chen2021stochastic} and \cite{gentil2017analogy}: let
\[dX_t = u_t^{\text{ref}}(X_t)dt +\sqrt{2}dW_t\]
with any fixed, arbitrary initial distribution, where we denote its marginal density as $\bar{q}_t$. Now rewriting the constraint in \eqref{eqn:BB}, we have (introduce the variable $v_t = b_t+u_t^{\text{ref}}-\nabla \log q_t$) 
\[\partial_t  q_t = -\nabla\cdot(q_t v_t),\; q_0=\rho_0 \, .\]
Similarly for the objective in \eqref{eqn:BB}, a change of variables gives
\begin{align*}
    &\int_{\mathbb{R}^d} \left[\frac{1}{4}\|b_t(x_t)\|^2  + J_t(x_t)\right] q_t(x_t) \, dx \\
    &= \int_{\mathbb{R}^d} \left[\frac{1}{4} \|v_t(x_t)-u_t^{\text{ref}}(x_t)+\nabla\log \bar{q}_t(x_t)+\nabla\log q_t(x_t)-\nabla\log \bar{q}_t(x_t)\|^2+ J_t(x_t)\right] q_t(x_t)dx
\end{align*}
Let $\bar{v}_t=u_t^{\text{ref}}-\nabla \log \bar{q}_t$. The cross term organizes into 
\[\frac{1}{2}\int \left[(v_t-\bar{v}_t)^\top(x) \nabla \log \frac{q_t}{\bar{q}_t}(x)\right] q_t(x) dx=\frac{1}{2}\frac{d}{dt}\text{KL}(q_t \Vert\bar{q}_t)\]
which upon time integrating over $[0,T]$ gives $\frac{1}{2}\text{KL}(q_T \Vert \bar{q}_T)$ if $q_T$ is not fixed and $\bar{q}_0=q_0=\rho_0$. After identifying $m_t=q_t v_t$, this gives \eqref{eqn:relative_fisher_obj}. 

If $u_t^{\text{ref}}=-\nabla V$ is not time-varying, and assuming we initialize $\bar{q}_0\propto e^{-V}=: \nu$ at stationary, then $\bar{q}_t=\nu$ is fixed, and we have
\[\frac{1}{2}\int \left[v_t(x)^\top\nabla \log \frac{q_t}{\nu}(x)\right] q_t(x) dx=\frac{1}{2}\frac{d}{dt}\text{KL}(q_t \Vert \nu)\, ,\]
therefore time-integrating the cross term gives us $\frac{1}{2}\text{KL}(q_T\Vert \nu)-\frac{1}{2}\text{KL}(\rho_0\Vert\nu)$, where the second part vanishes, which gives \eqref{eqn:simplified_objective}. 
\end{proof}

Before offering the details of our Eulerian solver from Section~\ref{sec:low_dim_eulerian}, we begin with some motivations.

If we split \eqref{eqn:relative_fisher_obj} as follows, one can view the highlighted red part as encoding the geometry (i.e., defining appropriate notion of shortest distance between $2$ distributions over the probability measure space) and the rest as the ``objective" function for the $(q_t)_t$ dynamics.
\begin{align}
&\min_{q_t,m_t} \;{\color{red}{\int_0^T \int_{\mathbb{R}^d} \left[\frac{1}{4}\left\|\frac{m_t}{q_t}(x_t)-(u_t^{\text{ref}}-\nabla \log\mathbb{P}_t^{\text{ref}})(x_t)\right\|^2\right] q_t(x_t)  dx dt}} \nonumber\\
&+\int_0^T \int_{\mathbb{R}^d}\left[\frac{1}{4}\left\|\nabla \log\frac{q_t}{\mathbb{P}_t^{\text{ref}}}(x_t)\right\|^2 +J_t(x_t)\right] q_t(x_t) dx dt+ \frac{1}{2}\int_{\mathbb{R}^d} q_T(x_T) \log \frac{q_T}{\mathbb{P}_T^{\text{ref}}}(x_T) \, dx_T \nonumber\\
&\text{s.t. } \; {\color{red}{\partial_t  q_t = -\nabla\cdot m_t,\; q_0=\rho_0}}  \nonumber
\end{align}

%Start with another change of variable, the following rewriting on \eqref{eqn:relative_fisher_obj} will simplify our problem and makes it closer to the formulation in \eqref{eqn:simplified_objective}. 

% \begin{align}
% &\min_{q_t,w_t} \;{\color{red}{\int_0^T \int_{\mathbb{R}^d} \frac{1}{4}\left\|\frac{w_t}{q_t}(x_t)\right\|^2q_t(x_t)  dx dt}} +\int_0^T \int_{\mathbb{R}^d} \left[\frac{1}{4}\left\|\nabla \log\frac{q_t}{\mathbb{P}_t^{\text{ref}}}(x_t)\right\|^2 +J_t(x_t)\right] q_t(x_t)  dx dt \nonumber\\
% & \quad + \frac{1}{2}\int_{\mathbb{R}^d} q_T(x_T) \log \frac{q_T}{\mathbb{P}_T^{\text{ref}}}(x_T) \, dx_T \nonumber\\
% &\text{s.t. } \;{\color{red}{\partial_t  q_t = -\nabla\cdot (w_t+q_t (u_t^{\text{ref}}-\nabla \log\mathbb{P}_t^{\text{ref}})),\; q_0=\rho_0}} 
% \end{align}
More specifically, since the first part depends on both $q_t,\dot{q}_t=m_t/q_t = v_t$, while the remaining part depends on $q_t$ only, for a given ground cost $L(\cdot,\cdot)$ over Euclidean space, if we define 
\begin{equation}
\label{eqn:W_L}
W_L(q_0,q_h):=\inf_{\rho_s} \inf_{v_s} \int_0^h \int_{\mathbb{R}^d} L(x_s,v_s)\rho_s(x_s)\, dx_s\, ds \quad \text{s.t.} \; \dot{\rho}_s = -\nabla \cdot(\rho_s v_s), \rho_0=q_0,\rho_h = q_h 
\end{equation}
then we may propose to iteratively solve
\begin{equation}
\label{eqn:WGF_variant}
q_{t+h}\leftarrow \arg\min_{q\in \mathcal{P}(\mathbb{R}^d)}\; W_L(q,q_{t})+h \cdot \underbrace{\int_{\mathbb{R}^d} \Big[\frac{1}{4}\Big\|\nabla \log\frac{q}{\mathbb{P}_{t+h}^{\text{ref}}}(x)\Big\|^2 +J_{t+h}(x)\Big] q(x) dx}_{F(q)}.
\end{equation}
for time stepping across $t\in[0,T]$ with time step $h$. 
% $s\in[0,1]$ is a fictitious time, only chosen for 
% , and unrelated to the physical time $t\in[0,T]$

Choosing $L(x_s,v_s)=\frac{1}{2}\|v_s(x_s)\|^2$ in \eqref{eqn:W_L} recovers the $\mathcal{W}_2$ OT distance used in the reduction \eqref{eqn:WGF} for \eqref{eqn:simplified_objective}. One can view \eqref{eqn:W_L} as a dynamic definition of a new metric. Here, again, we use the interpretation that given a curve of optimal densities for \eqref{eqn:relative_fisher_obj}, the optimization over $v_t$ (equivalently $m_t$), done in the red part, is simply solving an OT problem between the neighboring marginals, using a special cost. This elimination of $v_t$ (equivalently $m_t$) turns the problem into a single-variable optimization over the $(q_t)_t$ curve only and provides a good approximation when $h$ is small (we elaborate on this in Appendix~\ref{app:entropic_ot}).
% In some sense our approximation (and loss) only comes in when we trade the joint minimization over $\mathcal{P}(\mathbb{R}^d)^K$ for a sequence of $\mathcal{P}(\mathbb{R}^d\times\mathbb{R}^d)$ optimization over the neighboring pairs (c.f. Remark~\ref{rmk:multi-marginal-SB} for an elaboration on this point)

\subsection{Eulerian Solver}
Now we are ready to proceed with the approximation \eqref{eqn:WGF_variant} and the definition \eqref{eqn:W_L} to design our Eulerian solver. We preface with a definition.
\begin{definition}
An important convex set $\mathcal{K}$ that will be repeatedly used is 
\[\mathcal{K}:=\left\{(a,b)\in \mathbb{R}\times\mathbb{R}^d\colon a+\frac{1}{2}\|b\|^2\leq 0\right\}\, .\] 
The convex conjugate of the characteristic function of the set $\chi_\mathcal{K}^*(a,b) = \sup_{(a,b)\in \mathcal{K}}\{ar+b^\top m\} = \frac{\|m\|^2}{2r}$ for $(r,m)\in \mathbb{R}\times\mathbb{R}^d$ is its support function.
\end{definition}

%\begin{lemma}[Eulerian Solver]
%See in particular discretization $(5.1)-(5.2)$ therein. % this approach will scale as $N^d$ once we discretize
%\end{lemma}

%\begin{proof}
Without loss of generality, consider the first $q_0\rightarrow q_h$ step. We approximate the $W_L(q_0,q_h)$ term, similarly as done in \cite{carlier2023wasserstein} as
\begin{align*}
W_L(q_0,q_h)\approx \inf_{\bar{m}} \Big\{\int \frac{\|\bar{m}\|^2}{2 q_0}(x) &-\bar{m}^\top(x)(u_0^{\text{ref}}-\nabla \log\mathbb{P}_0^{\text{ref}})(x) \\
&+\frac{1}{2}\|u_0^{\text{ref}}(x)-\nabla \log\mathbb{P}_0^{\text{ref}}(x)\|^2 q_0(x)\, dx \colon q_h-q_0 = -\nabla\cdot \bar{m} \Big\}
\end{align*}
This relies on the fact that if $q_0,q_h$ are close, the momentum $m_s$ in the definition of $W_L(q_0,q_h)$ becomes almost $s$-independent, which we denote by $\bar{m}$ above.

%$\rho_1-\rho_0+\nabla\cdot(\bar{m})=0$. With $\bar{w}=\bar{m}-\rho_0(u_t^{\text{ref}}-\nabla \log\mathbb{P}_t^{\text{ref}})$, we solve for $\bar{w}$ and $\rho_1$.

 Now introduce notations that group known/fixed quantities together: let
\[g(x) = u_0^{\text{ref}}(x)-\nabla \log\mathbb{P}_0^{\text{ref}}(x), \; \; f_h(x)=2\cdot J_h(x)+\frac{h}{2}\|\nabla \log \mathbb{P}_h^{\text{ref}}(x)\|^2+h\cdot \Delta \log \mathbb{P}_h^{\text{ref}}(x)\]
and using the representation of Fisher information in \eqref{eqn:fisher_rewrite_2}, the problem reduces to solve the following convex problem (for some constant $C$)
\begin{align*}
\inf_{q_h,\bar{m}} \int &\frac{1}{2}\frac{\|\bar{m}\|^2}{q_0}(x)-\bar{m}^\top(x)g(x) +f_h(x) q_h(x)dx + C\\
&+\sup_{\phi_h,(\alpha,\beta)\in \mathcal{K}}\int h\cdot (\alpha(x)q_h(x)-\nabla\cdot\beta(x) q_h(x)) + \phi_h(x)(q_h(x)-q_0(x)+\nabla\cdot \bar{m}(x)) dx\, .
\end{align*}
It implies from optimality condition on $q_h,\bar{m}$ that
\[\bar{m}(x) = q_0(x)[g(x)+\nabla \phi_h(x)], \quad \alpha(x) = -\frac{1}{h}f_h(x)+\nabla\cdot\beta(x)-\frac{1}{h}\phi_h(x)\, .\]
Plugging it back in and flipping the sign, we end up with solving
%\begin{subequations}
\begin{equation}
\begin{aligned}
\inf_{\phi_h,\beta}\;\; & \frac{1}{2}\int_{\mathbb{R}^d} \|\nabla \phi_h(x)+g(x)\|^2 q_0(x) dx+\int_{\mathbb{R}^d} \phi_h(x) q_0(x) dx \label{eqn:composite-1}\\
&\text{s.t}\;\; \left(\nabla\cdot\beta(x)-\frac{1}{h}f_h(x)-\frac{1}{h}\phi_h(x),\beta(x) \right)\in \mathcal{K} \quad \forall x\in \mathbb{R}^d \, .
 % +\frac{\|\beta(x)\|^2}{2}\leq 0
  \end{aligned}
\end{equation}
%\end{subequations}
And the resulting updated $q_h,v_h$ can be recovered as
\begin{equation}
\label{eqn:q_on_grid}
q_h(x)\approx q_0(x)-\nabla\cdot(q_0(x)[\nabla\phi_h(x)+g(x)]) \quad\text{and}\quad v_h(x)\approx \nabla\phi_h(x)+g(x)\,.
\end{equation}
% Here $f(x),g(x)$ are functions of the reference process and/or $J(x;y)$ term that are known and considered fixed. 
Problem \eqref{eqn:composite-1} can now be solved using the Chambolle-Pock algorithm similar to what's done in \cite{carlier2023wasserstein}. As a splitting scheme for alternatively updating the 2 parts of the convex objective:
\[\min_z \; \Phi(z)+\Psi(\Lambda z)\]
we write for $z=(\phi_h,\beta)$ and $\Lambda z = (\nabla\cdot\beta-\frac{1}{h}\phi_h,\beta)$:
\[\Phi(z)=\frac{1}{2}\int_{\mathbb{R}^d} \|\nabla \phi_h(x)+g(x)\|^2 q_0(x) dx+\int_{\mathbb{R}^d} \phi_h(x) q_0(x) dx, \]
\[\Psi(\Lambda z)=\Psi(z_1,z_2)=\begin{cases}
0, &\; \text{if } (z_1(x),z_2(x))\in\mathcal{K}_{1/h\cdot f_h}, \text{for each } x \in \mathbb{R}^d \\
\infty, &\; \text{otherwise}
\end{cases}\; \]
where we overload the notation so $\mathcal{K}_c =\left\{(a,b)\in \mathbb{R}\times\mathbb{R}^d\colon a+\frac{1}{2}\|b\|^2\leq c\right\}$ for some known constant $c$.

In order to run the iterative update prescribed by Chambolle-Pock (for $k\geq 1$, some stepsize $\sigma,\delta > 0$, and initialization $y^0=u^0=z^0$)
\begin{equation}
\begin{aligned}
\label{update:CP}
y^{k+1} &= \text{prox}_{\sigma \Psi^*} (y^k+\sigma \Lambda u^k)\\
z^{k+1} &= \text{prox}_{\delta \Phi} (z^k-\delta \Lambda^* y^{k+1} )\\
u^{k+1} &= 2z^{k+1}-z^k
\end{aligned}
\end{equation}
 we only need to know the prox operator for $\Phi, \Psi$; since $\text{prox}_{\Psi^*} = I_d-\text{prox}_\Psi$ follows easily. The prox operator for $\Phi(z)$ involves solving a linear system on $\phi_h$, which returns $\text{prox}_{\delta \Phi}(\phi_0,\beta_0)=(\phi,\beta_0)$ where $\phi$ solves at each $x$ on the grid:
% \[\phi+\delta\cdot D_\tau^*((D_\tau \phi+g)q_0) = \phi_0-\delta\cdot q_0\]
\[\phi+\delta\cdot D_\tau^*(D_\tau \phi q_0) = \phi_0-\delta\cdot q_0-\delta\cdot D_\tau^*(g q_0)\, .\]
Above $D_\tau$ is a finite difference operator with grid spacing $\tau$ in each direction, and $-D_\tau^*$ approximates the divergence.
The other prox operator for $\Psi$ is a projection onto the set $\mathcal{K}$ that is separable across space (turns into a $1$D problem) and explicit \cite{benamou2000computational}. It remains to work out $\Lambda^*$ for running the update \eqref{update:CP}: since $\langle y,\Lambda z\rangle = \langle \Lambda^* y, z\rangle$, we have $\Lambda^*y = (-1/h \cdot y_1,-D_\tau y_1+y_2)$ in this case.

We have seen how to advance $t=0\rightarrow t=h$ for approximately solving \eqref{eqn:relative_fisher_obj} -- the whole process can be sequentially repeated until time $T$ since we have the density $q(x)$ explicitly computed (and maintained) on a fixed grid over $\mathbb{R}^d$ at each step (c.f. \eqref{eqn:q_on_grid}). %\qijia{in the last step ...}
%\end{proof}

% \begin{remark}
% In fact the same conceptual idea as in \eqref{eqn:WGF}, due to separability over $t$, can be applied in reverse time from $T\rightarrow 0$ in \eqref{eqn:WGF} as well if we have $\rho_T$ as constraint. It amounts to a different discretization since \eqref{eqn:simplified_objective} is a deterministic dynamics (in measure space). This suggests the possibility of doing a backward pass after the forward one for refinement, or use the parameterization from Appendix D of \cite{neklyudov2023computational} starting from both ends. 
% \end{remark}

\section{Additional Details for the Two Lagrangian Methods from Section~\ref{sec:lagrangian_methods}}
\label{app:sec_4}

In this section, we give the details for our two Lagrangian methods from Section~\ref{sec:lagrangian_methods} -- we begin with the Fisher Information Functional estimator used in Section~\ref{sec:push_forward_map} and Algorithm~\ref{alg:jko}.

\FIestimator*
\begin{proof}
Notice the optimization \eqref{eqn:fisher_rewrite_2} becomes essentially 1D and separable across grid points, if we approximate 
\[\hat{q}_t(x)=\frac{1}{n}\sum_{i=1}^n \delta_{x_i}\approx \frac{1}{n}\sum_{i=1}^n \delta_{x_i}(x) * \mathcal{N}(0,\sigma^2 I)=\frac{1}{n}\sum_{i=1}^n \mathcal{N}(x_i,\sigma^2 I)(x)\;\text{ for } \{x_i\}_{i=1}^n\sim q_t\] 
with a bandwidth $\sigma$, so
\[\nabla \hat{q}_t(x) \approx \frac{1}{n}\sum_{i=1}^n  -\frac{x-x_i}{\sigma^2} \mathcal{N}(x_i,\sigma^2 I)(x) \, .\]  
This makes $\nabla \hat{q}_t$, and therefore $R(\hat{q}_t)$ amenable to estimation given i.i.d samples from $q_t$. For solving \eqref{eqn:fisher_rewrite_1}, we sample $y_j$'s from Gaussians centered at $\{x_i\}_{i=1}^n$ where $x_i\sim q_t(x)$, and return 
\[\int \nabla \hat{q}_t(y) \beta(y) dy \approx \frac{1}{n}\sum_{i=1}^n \int -\frac{y-x_i}{\sigma^2} \mathcal{N}(x_i,\sigma^2 I)(y) \beta(y) dy\approx \frac{1}{n}\sum_{i=1}^n -\frac{y_j-x_i}{\sigma^2}\beta(y_j) ,  \] % \text{repeat $m$ times, average}
\[\int \alpha(y)\hat{q}_t(y) dy = \frac{1}{n}\sum_{i=1}^n \int \alpha(y) \mathcal{N}(x_i,\sigma^2 I)(y) dy\approx \frac{1}{n}\sum_{i=1}^n \alpha(y_j) \]
so estimating $R(\hat{q}_t)$ in turn becomes solving on the (random) $n$-point grid $y_i\sim  \mathcal{N}(x_i,\sigma^2 I)$
\begin{equation}
\label{eqn:random_FI_estimator}
R(\hat{q}_t)\approx \sup_{(\alpha_i,\beta_i) \in \mathbb{R}\times \mathbb{R}^d\in \mathcal{K}, \forall i}\;\; \frac{1}{n}\sum_{i=1}^n \left\{\left(-\frac{y_i-x_i}{\sigma^2}\right)^\top \beta_i+\alpha_i\right\} =\frac{1}{n}\sum_{i=1}^n \frac{\|y_i-x_i\|^2}{2\sigma^4}\, . 
\end{equation}
This can also be seen as a consequence of Stein's lemma: $\mathbb{E}[(X-\mu)^\top\beta(X)]=\mathbb{E}[\nabla\cdot \beta(X)]$ for $X\sim\mathcal{N}(\mu,I)$ on \eqref{eqn:fisher_rewrite_2}. In practice, a scheme that involves multiple samples at each $x_i$ point as $y_i^j\sim  \mathcal{N}(x_i,\sigma^2 I), j\in[m]$
\[R(\hat{q}_t)\approx \frac{1}{n}\sum_{i=1}^n \left(\frac{1}{m}\sum_{j=1}^m\frac{\|y_i^j-x_i\|^2}{2\sigma^4}\right)\]
maybe more accurate. 
\end{proof}

We'd like to emphasize that the estimator given in Lemma~\ref{lem:FI_estimator} is much simpler than what one would get by plugging in a density estimator $\hat{q}_t$ for $q_t$, and compute
\[R(q_t)\approx R(\hat{q}_t)=\frac{1}{2}\int \|\nabla \log \hat{q}_t(x)\|^2 \hat{q}_t(x) dx\]
directly. And this convenience crucially hinges on the variational representation \eqref{eqn:fisher_rewrite_2}.

\subsection{Missing Proofs for Section~\ref{sec:IPS}}
\label{app:IPS}

The core result on which our interacting particle system is based off is given below.
\ER*
\begin{proof}
We multiply through by $2$ to make the final expression cleaner. Writing out the Lagrangian for \eqref{eqn:simplified_objective}, we have using integration by parts:
\begin{align*}
\min_{q_t, v_t} \, \max_{\lambda_0,\theta_t}\; &\int_0^T \int \left[\frac{1}{2}\|v_t(X_t)\|^2+\frac{1}{2}\|\nabla\log \frac{q_t}{\nu}(X_t)\|^2+2J_t(X_t)\right]q_t(X_t) dX_t dt\\
&-\int_0^T \int (\partial_t \theta_t(X_t)\cdot q_t(X_t)+\nabla \theta_t(X_t)^\top v_t(X_t) q_t(X_t)) dX_t dt\\
&+\int \theta_T(X_T)q_T(X_T) dX_T-\int\theta_0(X_0)q_0(X_0) dX_0 \\
&+ \int \lambda_0(X_0)(q_0(X_0)-\rho_0(X_0)) dX_0 + \int q_T(X_T)\log\frac{q_T}{\nu}(X_T) dX_T
\end{align*}
Now the optimality condition for $v_t$ implies
\[v_t(X_t) = \nabla \theta_t(X_t)\]
and optimality condition for $q_t$ can be worked out using the first variation of $\delta I(q_t)$ from \cite{gianazza2009wasserstein}, where $I(q_t;\nu)$ is regarded as a potential energy involving the density:
\begin{equation}
\label{eqn:opt_cond_q}
\delta_q I(q_t;\nu)(X_t)+\frac{1}{2}\|v_t(X_t)\|^2+2J_t(X_t)-\partial_t \theta_t(X_t)-\nabla \theta_t(X_t)^\top v_t(X_t)=0
\end{equation}
which putting together with the previous equation yields the claim. Other parts are straightforward.

The continuity equation, which comes from the optimality condition for $\theta_t$, implies that along the optimal dynamics, particle moves with velocity field $\nabla \theta_t(X_t)$, which upon taking another time derivatives suggests the particles accelerate as
\[\ddot{X}_t=\nabla \frac{\partial \theta_t}{\partial t}(X_t)+\nabla^2 \theta_t(X_t)\nabla \theta_t(X_t)\, .\]
Direct calculation gives the equivalence to \eqref{eqn:Fisher_WGF} using the derived PDE optimality condition \eqref{eqn:opt_cond_q}. 
\end{proof}

Armed with this, we can simply run a deterministic $2$nd order ODE dynamics for simulating particle trajectories with the force field $-\nabla F(q_t)$ from Proposition~\ref{prop:euler_lagrange}, which consists of a linear part involving $J$ and a nonlinear part involving $I(q;\nu)$:
\begin{align*}
X_h= X_0+h\nabla \theta_0(X_0)-&\int_0^h\int_0^s \nabla F(q_t)(X_t)\, dt ds=X_0+h\nabla \theta_0(X_0)-\int_0^h \nabla F(q_t)(X_t)(h-t)\, dt \\
\nabla \theta_h (X_h) &= \nabla\theta_0(X_0)-\int_0^h \nabla F(q_t)(X_t) dt\, . 
\end{align*}
Upon time discretization it allows us to trace the optimal density curve.

% \begin{lemma}
%  However, for getting the $b_t^*$, one needs a final action matching done using the evolution of samples $\{X_t^i\}_{t\in[0,T], i\in[n]}$.
% \end{lemma}
% \begin{proof}

\paragraph{Kernelizing the vector field} 
With the force field depending on $q_t$, we resort to an interacting particle system for executing this proposal. This is again, where the relation \eqref{eqn:fisher_rewrite_2} for the Fisher Information functional becomes extremely useful, since na\"{\i}vely it is not clear how one may estimate terms such as $\nabla (\Delta\sqrt{q}/\sqrt{q})$. In particular \eqref{eqn:fisher_rewrite_2} shows that $\delta R(q_t)(x) = \alpha^*(x)-\nabla \cdot\beta^*(x)$ from Danskin's Theorem, where $\alpha^*(x)=-\frac{1}{2}\|\nabla \log q_t(x)\|^2, \beta^*(x)=\nabla\log q_t(x)$. Therefore the kernelized vector field evaluated at any $z\in\mathbb{R}^d$ is %(suppose kernel $\mathcal{K}$ is Gaussian)
\begin{align}
&[\mathcal{K}_{q_t}\nabla \delta R(q_t)](z) =\int \mathcal{K}(x-z)\nabla [\alpha^*(x)-\nabla \cdot\beta^*(x)]q_t(x) dx \nonumber\\
&= -\int \alpha^*(x) \nabla  (\mathcal{K}(x-z)q_t(x)) dx + \int \nabla\cdot \beta^*(x)  \nabla  (\mathcal{K}(x-z)q_t(x)) dx \nonumber\\
&= - \int \alpha^*(x) (\nabla \mathcal{K}(x-z) q_t(x)+\mathcal{K}(x-z) \nabla q_t(x)) dx - \int \beta^*(x)\nabla\cdot [\nabla  (\mathcal{K}(x-z)q_t(x))] dx \nonumber\\
&= - \int \alpha^*(x) (\nabla \mathcal{K}(x-z) q_t(x)+\mathcal{K}(x-z) \nabla q_t(x)) dx \label{eqn:KWGF}\\
& - \int\beta^*(x) \nabla\cdot \nabla \mathcal{K}(x-z) q_t(x)+ 2\beta^*(x) \nabla \mathcal{K}(x-z)^\top \nabla q_t(x)dx-\int\beta^*(x) \mathcal{K}(x-z)\nabla\cdot(\nabla q_t(x)) dx \nonumber
\end{align}
This suggests an implementation as 
\begin{align*}
X_{t+h}^i&=X_t^i +h V_t(X_t^i) + \frac{h^2}{2} \cdot \left(\mathcal{K}\nabla \delta R(\hat{q}_t)(X_t^i) +  2\nabla J_t(X_t^i)\right),\;\; i\in[n]\\
V_{t+h}(X_{t+h}^i)&=V_t(X_t^i) + h \cdot \left(\mathcal{K}\nabla \delta R(\hat{q}_t)(X_t^i)+2\nabla J_t(X_t^i)\right)
\end{align*}
for 
\[\hat{q}_t(x)=\frac{1}{n}\sum_{i=1}^n \delta_{x_i}\approx \frac{1}{n}\sum_{i=1}^n \mathcal{N}(x_i,\sigma^2 I)(x)\;\text{ with } x_i\sim q_t\]
so the particle system with $n$ particles will approximately follow the optimal $(q_t)_{t\leq T}$ curve with stepsize $h$. 

In \eqref{eqn:KWGF}, earlier calculation from \eqref{eqn:random_FI_estimator} shows that e.g., $\int \alpha^*(x) \nabla \mathcal{K}(x-z) q_t(x)dx$ can be approximated as
\begin{equation}
\label{eqn:estimate_kernel_VF}
\frac{1}{n}\sum_{i=1}^n \underbrace{\int \alpha^*(x)\nabla \mathcal{K}(x-z)\mathcal{N}(x_i,\sigma^2 I)(x) dx}_{\mathbb{E}_{y\sim \mathcal{N}(x_i,\sigma^2 I)}[\alpha^*(y)\nabla\mathcal{K}(y-z)]}\approx \frac{1}{n}\sum_{i=1}^n \left(\frac{1}{m}\sum_{j=1}^m -\frac{\|y_i^j-x_i\|^2}{2\sigma^4} \nabla\mathcal{K}(y_i^j-z)\right)
\end{equation}
since on a random grid $\{y_i^j\}_j\sim \mathcal{N}(x_i,\sigma^2 I)$ with $x_i\sim q_t$, we showed
\[\alpha^*(y_i^j)=-\frac{\|y_i^j-x_i\|^2}{2\sigma^4},\;\; \beta^*(y_i^j)=\frac{x_i-y_i^j}{\sigma^2}\, .\]
The only term that's not in terms of $q_t,\nabla q_t$ in \eqref{eqn:KWGF} is the last one --  the rest can be similarly estimated as above. However, one can show that the divergence term also takes a rescaled Gaussian form: 
\begin{align*}
\nabla\cdot \nabla \hat{q}_t(x) &= \frac{1}{n}\sum_{i=1}^n  -\nabla\cdot\left(\frac{x-x_i}{\sigma^2} \mathcal{N}(x_i,\sigma^2 I)(x) \right)\\
&= \frac{1}{n}\sum_{i=1}^n -\frac{d}{\sigma^2} \mathcal{N}(x_i,\sigma^2 I)(x)+\frac{1}{n}\sum_{i=1}^n \frac{\|x-x_i\|^2}{\sigma^4} \mathcal{N}(x_i,\sigma^2 I)(x) \, ,%\\
%&= -\frac{d}{\sigma^2}\hat{q}_t(x) 
\end{align*}
which makes the last term expressible as an expectation w.r.t a Gaussian mixture density -- hence readily estimable provided samples from $\{y_i^j\}_j\sim \mathcal{N}(x_i,\sigma^2 I)$ for $x_i\sim q_t$, as done in \eqref{eqn:estimate_kernel_VF}. 

The additional KL part needed for the update at $T$ can be kernelized as 
\begin{align*}
\int\mathcal{K}(x_T-z) \nabla &\log\frac{q_T}{\nu}(x_T) q_T(x_T) \, dx_T =\int -\nabla_1\mathcal{K}(x_T-z) q_T(x_T) + \mathcal{K}(x_T-z) \nabla V(x_T)q_T(x_T)\, dx_T\,,
\end{align*}
followed by replacing $q_T$ with $\hat{q}_T$. This is similar to what's done in SVGD \cite{liu2017stein}.

% \end{proof}

Compared to Algorithm~\ref{alg:jko} that leverages push-forward maps, this particular approach, although also evolves a density over time through space, does not require NN training. Other particle methods for implementing dynamics in Wasserstein space can be found in e.g., \cite{carrillo2015numerical}. We note that most of the literature focuses on Wasserstein Gradient Flow (\`a la JKO), whereas here we consider $2$nd-order dynamics.

\begin{remark}
% Redefining $q\mapsto q/\nu$ with respect to a reference measure $\nu\propto e^{-V}$ gives the relative Fisher functional in \eqref{eqn:simplified_objective}, whose first variation is 
% \[2\frac{\Delta \sqrt{q}}{\sqrt{q}}-2\frac{\Delta\sqrt{\nu}}{\sqrt{\nu}}=2\frac{\Delta \sqrt{q}}{\sqrt{q}}-\left(\frac{1}{2}\|\nabla V\|^2-\Delta V\right)\] 
% The additional part does not incur significant compute for
% \[\ddot{X}_t = \nabla\left(2\frac{\Delta \sqrt{q}}{\sqrt{q}}-\left(\frac{1}{2}\|\nabla V\|^2-\Delta V\right)+J_t\right)\]
% % \[\partial_t q + \nabla\cdot\left(q\nabla\left(2\frac{\Delta \sqrt{q}}{\sqrt{q}}-\left(\frac{1}{2}\|\nabla V\|^2-\Delta V\right)+J_t\right)\right) = 0\]
% with access to $\nabla V$. 
%It is clear in this case, as $T\rightarrow \infty$, asymptotically (assuming $\nu$ is log-concave) if there's
%
A rewriting of the acceleration field makes it clear what the driving forces behind the dynamics are: 
\[\ddot{X}_t = \nabla_{\mathcal{W}_2} \left(\|\nabla_{\mathcal{W}_2} \text{KL}(q\Vert \nu)\|_q^2+\mathbb{E}_q[J]  \right)(X_t)\, .\]
Since the squared norm of the $\mathcal{W}_2$-gradient for the KL divergence is the relative Fisher Information, we expect the particles to undergo rapid speed change when either $q$ differs significantly from $\nu$ or if the external potential $J$ is exerting large force.  
%\qijia{Second order calculus $\approx \nabla^2_{\mathcal{W}_2} \text{KL}(q\Vert \nu) \nabla_{\mathcal{W}_2} \text{KL}(q\Vert \nu)$? }
\end{remark}

\section{Alternative Decompositions \& Approximations with Entropic-OT}
\label{app:entropic_ot}

In this section, we give additional details on Section~\ref{sec:entropic_OT}. Our earlier reductions operate in the space of probability measures $\mathcal{P}(\mathbb{R}^d)$ endowed with a Wasserstein $\mathcal{W}_2$ geometry. We explore other geometries here. 

As mentioned in the main text, with the grouping \eqref{eqn:entropic_W2}, we can afford a larger time-stepping, i.e., solve between each adjacent pair of observation time $\{t_k\}_{k=1}^K$ -- this is different from the $o(h^2)$ approximation error we incur when using $\mathcal{W}_2$. For reversible Brownian motion $W$ as reference, it becomes iteratively solving between $2$ observations:
\begin{equation}
\label{eqn:eOT_GF_1}
q_{t_{i+1}} \leftarrow \arg\min_{q\in \mathcal{P}(\mathbb{R}^d)}\;  \mathcal{W}_{2,t_{i+1}-t_i}^2(q,q_{t_i})+\underbrace{2\int J_{t_{i+1}}(x) q(x) dx}_{\bar{F}(q)}
\end{equation}
with entropy-regularized quadratic $\mathcal{W}_2$ distance
\[\mathcal{W}_{2,t_{i+1}-t_i}^2(q,q_{t_i}) = \min_{\pi^i \in \Pi(q_{t_i}, q)}\int \frac{1}{2(t_{i+1}-t_i)}\|x-y\|^2d\pi^i(x,y)+\text{KL}(\pi^i\Vert q_{t_i}\otimes q)\]
as the distance-measuring metric and the linear likelihood term as the objective functional $\bar{F}$, which is arguably simpler than that based on \eqref{eqn:WGF}
\[q_{t+h}\leftarrow \arg\min_{q\in \mathcal{P}(\mathbb{R}^d)}\; \frac{1}{2h} \mathcal{W}_2^2(q,q_{t})+h\cdot \underbrace{\int  \frac{1}{2}\|\nabla \log \frac{q}{\nu}(x)\|^2 q(x) dx +2\int J_{t+h}(x) q(x) dx}_{F(q)}\]
% $2\int J(x) dq(x)+ \int \log q_T dq_T$ entropy 
% \[\frac{1}{2h} \left(h\int_0^h \int_{\mathbb{R}^d}  q_t(x)\|v_t(x)\|^2 dxdt \right)+h\cdot \int  \frac{1}{2}\|\nabla \log \frac{q_h}{\nu}(x_h)\|^2 q_h(x_h) dx +2\int J_h(x_h) q_h(x_h) dx\]
involving $I(q;\nu)$. 
% since in between the solution is given by entropic interpolation

\paragraph{Eulerian Method} Using ideas from \cite{peyre2015entropic} gives a proximal splitting first order method in low dimension for \eqref{eqn:eOT_GF_1}. It results in a sequence of Sinkhorn-like updates on the dual potentials. But being confined to a fixed grid, such methods will run into scalability problem in high dimension. % In the setting when $J_t$ is sparse, most of the steps are  (recall $\mathcal{W}_{2,\epsilon}^2$ is biased therefore won't return the previous iterate $q_{t_i}$ as solution)

 % The iterates \eqref{eqn:eOT_GF} with regularized-OT distance, although having a simpler objective, does not admit push-forward map solution or modeling as PDE dynamics. Hence it's not clear if there are simple Lagrangian methods for implementing the $\{\rho_t\}_t$ dynamics in high dimension. 
 
\paragraph{Lagrangian Method} One candidate here for implementing e.g., $t=t_{i}\rightarrow t_{i+1}$ update in \eqref{eqn:eOT_GF_1} in high dimension can be: 
%
% parametrize $\rho(x)\colon\mathbb{R}^d \rightarrow \mathbb{R}$ with NN directly gives on the density level
% \[\rho(x) \leftarrow \rho(x) - \eta\cdot\left(\frac{1}{h}\phi^\rho(x)-\frac{\Delta\sqrt{\rho(x)}}{\sqrt{\rho(x)}} + J_t(x) \right)\]
%
let $\hat{q}_{t_i}=\frac{1}{n}\sum_{j=1}^n \delta_{x_j^{t_i}}$ be given as $n$ particles distributed as $q_{t_i}$, since the first variation of $q\mapsto \mathcal{W}_{2,t_{i+1}-t_{i}}^2(q,q_{t_{i}})$ is the optimal dual potential $ \varphi$ for the coupling, for iteration $k\geq 1$ perform
\begin{equation}
\begin{aligned}
\label{eqn:entropic_IPS}
x_j^{t_{i+1}}[k+1]&\leftarrow x_j^{t_{i+1}}[k] - \eta\cdot\left(\frac{1}{t_{i+1}-t_{i}}\nabla\varphi^{\hat{q}_{t_{i+1}}^k}(x_j^{t_{i+1}}[k]) + 2\nabla J_{t_{i+1}}(x_j^{t_{i+1}}[k]) \right), \quad j\in[n]\\
\hat{q}_{t_{i+1}}^k &= \frac{1}{n}\sum_{j=1}^n \delta_{x_j^{t_{i+1}}[k]}\; \text{ with }\; \hat{q}_{t_{i+1}}^1 = \hat{q}_{t_i}\, .
\end{aligned}
\end{equation}
In the above we expect $\hat{q}_{t_{i+1}}^k\rightarrow q_{t_{i+1}}$, the optimal solution of \eqref{eqn:eOT_GF_1}, as $k \rightarrow \infty$. To get $\varphi^{\hat{q}}$ needed in \eqref{eqn:entropic_IPS}, one can solve an empirical Entropy-regularized OT problem, using its dual formulation: the optimal dual potential $\varphi^{\hat{q}_{t_{i+1}}^k}$ for the marginal $\hat{q}_{t_{i+1}}^k$ (with the other being fixed at $\hat{q}_{t_i}$) solves
\[\arg\max_{\varphi,\psi}\; \sum_{j=1}^n \varphi(X_j)+\psi(Y_j)-(t_{i+1}-t_{i}) \cdot \exp\left(\frac{\varphi(X_j)+\psi(Y_j)-\frac{\|X_j-Y_j\|^2}{2}}{t_{i+1}-t_{i}}\right), \quad \text{for } X_j \sim \hat{q}_{t_{i+1}}^k, Y_j\sim \hat{q}_{t_i}\]
From this we take the optimal $\{\psi^*(Y_j)\}_{j=1}^n$, and set 
\[\varphi^*(x)=-(t_{i+1}-t_{i})\log\left(\frac{1}{n}\sum_{j=1}^n e^{\frac{\psi^*(Y_j)-\|x-Y_j\|^2/2}{t_{i+1}-t_{i}}}\right)\]
to be the soft $c$-transform, after which we can take its derivative to get (with this expression $x\in \mathbb{R}^d$ can be evaluated anywhere)
\[\nabla\varphi^{\hat{q}_{t_{i+1}}^k}(x) = \nabla \varphi^*(x) = \frac{\sum_{j=1}^n (x-Y_j) e^{\frac{\psi^*(Y_j)-\|x-Y_j\|^2/2}{t_{i+1}-t_{i}}}}{\sum_{j=1}^n e^{\frac{\psi^*(Y_j)-\|x-Y_j\|^2/2}{t_{i+1}-t_{i}}}}\]
in order to iterate the algorithm \eqref{eqn:entropic_IPS}. It is known that the optimal potentials $\varphi^*,\psi^*$ are related through soft $c$-transforms at optimality.

The procedure described above requires evolving $n$ interacting particles over the iteration $k \geq 1$ and between each pair of time $(t_{i},t_{i+1})\in[T]$. Each update $k$ requires running the Sinkhorn algorithm to find the optimal dual potential $\{\psi^*(Y_j)\}_{j=1}^n$ for the updated empirical density $\hat{q}_{t_{i+1}}^k$. Repeatedly running such subroutines may not be cheap computationally.  % Entropic Action Matching can be used to recover the drift in the SDE in \eqref{eqn:BB}. 

But as already commented in Remark~\ref{rmk:entropic_OT}, the more serious concern here is that decomposition \eqref{eqn:eOT_GF_1} cannot be generalized when we don't know the transition kernel of the reference process -- the cost function is fully determined by it.

\paragraph{Multi-marginal Schr\"odinger Bridge}
%\label{rmk:multi-marginal-SB}
More generally, using \cite[Lemma 3.4]{benamou2019entropy} on the $\text{KL}(\mathbb{Q}^\theta\Vert\mathbb{P}^{\text{ref}})$ term in the second equation of \eqref{eqn:loss_path}: 
\[\arg\min_{\mathbb{Q}^\theta:\mathbb{Q}^\theta_0=\rho_0}\;\; \mathbb{E}_{\mathbb{Q}^\theta}[J(x;y)]+\text{KL}(\mathbb{Q}^\theta\Vert\mathbb{P}^{\text{ref}})\, ,\]
we have that the objective over $\mathcal{P}(\mathcal{C}([0, T], \mathbb{R}^d))$ above can equivalently be written as %(assuming $\rho_{t_1}$ fixed)
\begin{align*}
&\arg\min_{\{q_{t_i}\}_i}\;\; \sum_{i=1}^K \int J_{t_i}(x) q_{t_i}(x) dx  + \sum_{i=1}^{K-1} \text{KL}(q_{t_i,t_{i+1}}\Vert R_{t_i,t_{i+1}})-\sum_{i=2}^{K-1}\text{KL}(q_{t_i}\Vert R_{t_i})\\
&\Leftrightarrow \arg\min_{\{q_{t_i}\}_i}\;\; \sum_{i=1}^K \int J_{t_i}(x) q_{t_i}(x) dx  + \sum_{i=1}^{K-1} \text{KL}(\pi^{*}(q_{t_i},q_{t_{i+1}})\Vert R_{t_i,t_{i+1}}q_{t_i}\otimes q_{t_{i+1}})+H(q_{t_i})+H(q_{t_{i+1}})\\
& -\sum_{i=2}^{K-1}H(q_{t_i}) + \sum_{i=2}^{K-1}\mathbb{E}_{q_{t_i}}[\log\nu]\\
& \Leftrightarrow \arg\min_{\{q_{t_i}\}_i}\;\; \sum_{i=1}^K \int J_{t_i}(x) q_{t_i}(x) dx  \quad\quad (\text{assuming Brownian motion reference})\\ 
&+\sum_{i=1}^{K-1} \underbrace{\min_{\pi^i \in \Pi(q_{t_i}, q_{t_{i+1}})}\int \frac{1}{2(t_{i+1}-t_i)}\|y-x\|^2d\pi^i(x,y)+\text{KL}(\pi^i\Vert q_{t_i}\otimes q_{t_{i+1}})}_{\text{SB}(q_{t_i},q_{t_{i+1}})}+\sum_{i=1}^K H(q_{t_i})\, .
\end{align*}
This implies that the posterior sampling problem we consider can be exactly reduced to a multi-marginal SB over the $K$ observation times (i.e., joint minimization over $\{q_{t_i}\}_{i=1}^K$), with additional parts of the loss involving simple functionals on a single marginal only. However, beyond the simple cases of e.g., Ornstein–Uhlenbeck or Brownian motion, it is not possible to specify the cost in $\text{SB}(q_{t_i},q_{t_{i+1}})$ (transition probability density of the reference measure unknown generally). Such approach, in the case of the Wiener reference, is adopted in \cite[Theorem 3.1]{chizat2022trajectory} where the authors propose a mean-field Langevin algorithm to solve this reduced problem over the joint $\mathcal{P}(\mathbb{R}^d)^K$ space, after which one samples from the conditional Brownian bridge between each neighboring pair to reconstruct the entire trajectory. Such a simulation is also challenging for the non-linear reference process that we consider. 
%\end{remark}

% \qijia{Check \url{https://arxiv.org/pdf/1810.02733}}
% Comment on terminal $\rho_T$ update.

\newpage 
\section{Pseudocode for Algorithms}
\label{app:pseudocode}

The algorithm we use for learning the SDE drift $\{b_s(X_t,t)\}_{s\in [0,1]}$ from Section~\ref{sec:track_b_transport} is provided below. This is based on controlled equilibrium dynamics since it seeks to impose the marginal density $(\pi_s)_{s\in [0,1]}$ exactly. % and updating the corresponding SDE trajectories 

\begin{algorithm}[htbp] % h!
\caption{Controlled transport from prior $\exp(-I(x))$ to posterior $\exp(-I(x)-J(x))$}
\label{alg:transport_b}
	\begin{algorithmic}[1]
		\STATE \textbf{Input:} Initial samples from prior $\{X_t\}_{t\leq T}\in\mathbb{R}^{d\times T/\delta_t}\sim \pi_0\propto \exp(-I(x))$ with $X_0\sim\rho_0$, potential/reference drift $-\nabla V$, likelihood $J$ over path, ensemble size $K$, anneal stepsize $\delta_s$
        \STATE \textbf{Input:} NN parameterizing $\{\partial b_t^s/\partial s\}_t$: $\phi^\theta(x_t^s,t,s)\colon\mathbb{R}^d\times [0,T]\times [0,1]\mapsto\mathbb{R}^d$ 
		\STATE Set $A_0 = 0, b_0(\cdot)=-\nabla V(\cdot)$
		\FOR{ $s = 0$ \TO $s = 1$ }
        \STATE Simulate until time $T$ for $K$ independent walkers:
        \begin{equation}
        \label{eqn:em}
        X_{t+1}= X_t+\delta_t \cdot b_s(X_t,t) +\sqrt{2\delta_t}\cdot \zeta_t, \;X_0\sim\rho_0, \zeta\sim\mathcal{N}(0,I)  % -\delta_t \cdot \nabla V(X_t)
        \end{equation}
        \textit{Optional:} Anneal w.r.t $\pi_s(X) \propto\exp(-I(X)-sJ(X))$ using SPDE \eqref{eqn:langevin-spde}-\eqref{eqn:spde_general} for a few steps, starting from trajectories got from \eqref{eqn:em}, output $\{X_t\}_{t\leq T}$ with $X_0\sim\rho_0$  
		\STATE Compute quantities evaluated on the trajectory
        \[\bar{J}^s=\frac{1}{K}\sum_{k=1}^K J(\{X_t^k\}_{t\leq T})\]
        \[h_\theta^s(\{X_t^k\}_{t\leq T})=\sum_{t=0}^{T/\delta_t}-\frac{\delta_t}{2}\cdot \left(b_s(X_t^k,t)-\frac{X_{t+\delta_t}^k-X_t^k}{\delta_t}\right)^\top \phi^\theta(X_t^k,t,s)-\frac{\delta_t}{4}\nabla\cdot \phi^\theta(X_t^k,t,s)\]
        \[\bar{h}_\theta^s = \frac{1}{K}\sum_{k=1}^K h_\theta^s(\{X_t^k\}_{t\leq T})\]
        Solve for 
    \[\phi^\theta \leftarrow \min_\theta \quad  \frac{1}{K}\sum_{k=1}^K \left(h_\theta^s(\{X_t^k\}_{t\leq T})-\bar{h}_\theta^s+J(\{X_t^k\}_{t\leq T})-\bar{J}^s\right)^2\]
    \STATE Update the drift (can be evaluated any $X_t,t$)
    \[b_{s+\delta s}(\cdot) \approx b_s(\cdot )+\phi^\theta(\cdot, s)\cdot \delta s\]
        \STATE Calculate 
        \[A_{s+\delta_s}=A_s-\delta s \cdot \bar{J}^s\]
		\ENDFOR
		\RETURN $ \exp(A_1)\approx Z_1/Z_0$ as normalizing constant estimator for $Q(x)$ and $b_1(\cdot)$
	\end{algorithmic}	
\end{algorithm}

%\newpage
 We use a particular implicit discretization scheme for the $s$ variable advocated in \cite{beskos2008mcmc} below for the non-equilibrium SPDE dynamics from Section~\ref{sec:Non-Equilibrium}. This is an extension of Jarzynski-style result to path space -- the algorithm, for simplicity, is written for TPS in this case. The additional Metropolis-Hastings adjustment can be used to eliminate the time-step error, as shown in Lemma~\ref{lem:Jarzynski}.

 %This may be seen as a way to eliminate the time-step error in the estimate of the free energy diﬀerence upon resorting to a time-inhomogeneous Metropolis-Hastings dynamics, instead of plain discretizations of the underlying diﬀusion process.

\begin{algorithm}[htbp] % h!
\caption{Annealed SPDE from Brownian Bridge + MH adjustment for TPS with weighting (built on \cite{beskos2008mcmc} 
 with path $x(t,s)$ for $(s,t)\in (0,\infty)\times [0,T]$)}
\label{alg:spde}
	\begin{algorithmic}[1]
		\STATE \textbf{Input:} Initial Brownian bridge path $x(t,0)\in\mathbb{R}^{d\times T/\delta_t}\sim \pi_0$ with $x(0,0)=A, x(T,0)=B$, potential/reference drift $-\nabla V$, ensemble size $K$, anneal stepsize $\delta_s$, physical time stepsize $\delta_t$
		\STATE Set $A_0 = 0$
		\FOR{ $s = 0$ \TO $s = 1$ }
        \STATE Set the current potential to be $u=-s\nabla V$
		\STATE Using the discretization in physical time: for $A$ the discrete Laplacian matrix
        \[A = \frac{1}{\delta_t^2}
        \begin{bmatrix}
        -2& 1& & & & \\
         1 &-2 & 1 & & &\\ &&&\ddots &&\\ 
        && & 1 & -2 & 1\\
        &&&& 1 & -2
        \end{bmatrix}\]
        \[\frac{x(t,s)}{ds} = \frac{1}{2} Ax(t,s)-\frac{1}{2}u(x)\nabla u(x)-\frac{1}{2}\nabla(\nabla\cdot u(x))+\sqrt{\frac{2}{\delta t}}\frac{\partial W}{\partial s}\]
	We update with discretization in algorithmic time as (for $\zeta\sim\mathcal{N}(0,I)$):
        \begin{align*}
        \underbrace{(I-\frac{1}{4}\delta_s A)}_{L} x(t,s+\delta_s) &= \underbrace{(I+\frac{1}{4}\delta_s A)}_{R} x(t,s)-\underbrace{\frac{1}{2}u(x(t,s))\nabla u(x(t,s))\delta_s-\frac{1}{2}\nabla(\nabla\cdot u(x(t,s)))\delta_s}_{M(x(t,s))\delta_s}\\
        &+\sqrt{\frac{2\delta_s}{\delta_t}}\zeta
        \end{align*}
        with boundary condition $x(0,s+\delta_s)=A, x(T,s+\delta_s)=B$.
        % \[Lx(t,s+\delta_s)=Rx(t,s)-M(x(t,s))\delta_s+\sqrt{\frac{2\delta_s}{\delta_t}}\zeta\]
        \STATE MH adjustment: compute transition density
        \[q(s\rightarrow s+\delta_s)\propto\exp\left(-\frac{\delta_t}{4\delta_s}\|Lx(t,s+\delta_s)-Rx(t,s)+M(x(t,s))\delta_s\|^2\right)\]
        Accept with probability (otherwise stay at $x(t,s)$, i.e., $x(t,s+\delta s)=x(t,s)$)
        \[\alpha = \min\left(1,\frac{\pi(x(t,s+\delta_s))q(x(t,s+\delta_s)\rightarrow x(t,s))}{\pi(x(t,s)) q(x(t,s)\rightarrow x(t,s+\delta_s))}\right)\]
        where $\pi(x(t,s)) \propto $
        \[\exp\Bigg(-s\cdot\underbrace{\sum_{t=0}^{T/\delta t}\left(\frac{1}{4}\|-\nabla V(x(t,s))\|^2+\frac{1}{2}\nabla\cdot -\nabla V(x(t,s))\right)\delta_t}_{\tilde{J}(x)}+\underbrace{\frac{\delta t}{4} \sum_{i=1}^{d} \left(x(t,s)_i^\top A x(t,s)_i\right)}_{-\tilde{I}(x)}\Bigg)\]
        for $x(0,s)=A, x(T,s)=B$. The first part is discretization of \eqref{eqn:TPS_anneal}. Second part for \eqref{eqn:prior_bb}.
        \STATE Calculate weight
        \[A_{s+\delta_s}=A_s-\delta s \cdot \tilde{J}(x(t,s+\delta_s))\]
		\ENDFOR
        \STATE Repeat Line 2 - Line 8 for $K$ times to get $K$ independent estimator $A_1^1,A_1^2,\dots, A_1^K$
		\RETURN $\frac{1}{K}\sum_{k=1}^K \exp(A_1^k)\approx Z_1/Z_0$ as normalizing constant estimator for $Q(x)$
	\end{algorithmic}	
\end{algorithm}

% For the last KL term in \eqref{eqn:simplified_objective}, we can simply estimate with
% \[\int q_T(x_T) \log\frac{q_T}{\nu}(x_T) dx_T \approx \frac{1}{n}\sum_{i=1}^n \log \frac{1/n}{\nu(x_i)} \]
% for samples $\{x_i\}_{i=1}^n \sim q_T$.

\newpage
The Lagrangian method from Section~\ref{sec:push_forward_map} for approximately following the Wasserstein dynamics is parametrized by a sequence of push-forward maps, and outlined below. The randomized $I(q;\nu)$ estimator here is built with Lemma~\ref{lem:FI_estimator}. Note that the algorithm only requires knowledge of $\nabla V, \nu$.

\begin{algorithm}[htbp] % h!
\caption{JKO with pushforward map $\nabla \phi^\theta$ for solving \eqref{eqn:simplified_objective} based on the proposal \eqref{eqn:jko_map}}
\label{alg:jko}
	\begin{algorithmic}[1]
		\STATE \textbf{Input:} Initial samples $\{X_0^i\}_{i=1}^n \in\mathbb{R}^{d}\sim  \rho_0$, potential/reference drift $-\nabla V$ with stationary distribution $\nu=\rho_0=q_0$, likelihood $J_t$ over path, stepsize $h$, small $\sigma^2, m$
        \STATE \textbf{Input:} Input-convex NN parameterizing pushforward map $\phi^\theta(x_t,t)$: $\mathbb{R}^d\times [0,T]\mapsto\mathbb{R}$ 
		\FOR{ $t = 0$ \TO $t = T$ }
%        \STATE Sample $X_t^j \sim \text{Unif}(X_t^1,\dots, X_t^n)$ for $n$ times. For each $X_t^j$, sample from $Y_t^j\sim \mathcal{N}(X_t^j,\sigma^2 I)$.
\STATE Given current $\phi^\theta(\cdot, t)$, and samples $\{X_t^i\}_{i=1}^n \sim q_t$, let $X_{t+h}^i = \nabla \phi^\theta (X_t^i,t)$
         \STATE Compute the following quantity on the $n$ particles:
     \[\mathcal{L}^\theta(\{X_{t+h}^i\})=\frac{1}{n}\sum_{i=1}^n \left(\Delta\log \nu(X_{t+h}^i)+\frac{1}{2}\|\nabla V(X_{t+h}^i)\|^2\right)\]  % -\delta_t \cdot \nabla V(X_t)
      With the current samples $\{X_{t+h}^i\}_{i=1}^n$ generated from $\phi^\theta(\cdot, t)$, draw $Y_i^j\sim \mathcal{N}(X_{t+h}^i,\sigma^2 I)$ for $j\in[m]$ times and compute
            %solve for $n$ separable problems ($\zeta_j\sim\mathcal{N}(0,\sigma^2 I)$):
        \[R^\theta(X_{t+h}^i)=  \frac{1}{m}\sum_{j=1}^m\frac{\|Y_i^j-X_{t+h}^i\|^2}{2\sigma^4} \]
        % \sup_{(\alpha_j,\beta_j) \in \mathbb{R}\times \mathbb{R}^d\in \mathcal{K}}\;\;  \left(\frac{\nabla\phi^\theta(X_t^j)+\zeta_j-\nabla\phi^\theta(X_t^j)}{\sigma^2}\right)^\top \beta_j+\alpha_j
		\STATE Solve for 
    \begin{equation}
    \label{eqn:loss_jko}
    \phi^\theta(\cdot,t) \leftarrow \arg\min_\theta \; \frac{1}{2h^2}\frac{1}{n}\sum_{i=1}^n \left(X_{t+h}^i-X_t^i\right)^2+\frac{1}{n}\sum_{i=1}^n \left[ R^\theta(X_{t+h}^i)+ \frac{2}{h} J_{t+h}(X_{t+h}^i)\right]+ \mathcal{L}^\theta(\{X_{t+h}^i\})
    \end{equation}
     \IF {$t=T-h$}
\STATE Add additional term $\frac{1}{n h}\sum_{i=1}^n \log \frac{1/n }{\nu(X_T^i)}$ to the loss \eqref{eqn:loss_jko} with samples $X_T^i =\nabla \phi_{T-h}^\theta(X_{T-h}^i)$
    \ENDIF 
    \STATE Update the $n$ particles with the output $\phi_t^\theta$ from \eqref{eqn:loss_jko}:
    \[X_{t+h}^i = \nabla \phi^\theta (X_t^i,t) \;\; \text{so}\;\; \{X_{t+h}^i\}_{i=1}^n\sim q_{t+h} \]
		\ENDFOR
		\RETURN $\nabla \phi^\theta_t(\cdot)$ for all $t\in[0,T]$
	\end{algorithmic}	
\end{algorithm}

\clearpage
\section{Numerics}
\label{sec:numerics}

% \qijia{TODOs:
% \begin{itemize}
% \item Both the case of Brownian motion and OU as reference for TPS have analytical answers (\eqref{eqn:BB_SDE} and below) that we can compare against.
% \item A lower-d ($\leq 10$) double-well TPS (annealing from both prior SDE and Brownian bridge), with our NN annealing+SPDE-in-between approach (i.e., Algorithm \ref{alg:transport_b}) and report normalizing constant, possibly with other statistics
% \item  Compare the performance of this TPS example with non-equilibrium dynamics (i.e, Algorithm \ref{alg:spde}). The only additional part we'll have to add is MH and then accumulate weight in line 7, but other that the SPDE part should be similar
% \item A general Bayesian posterior example (stationary OU as reference) with Algorithm \ref{alg:jko} for multiple observations (one at time $t=T$, can skip line 7-9 in this case)
% \end{itemize}
% In the case of OU, Doob's $h$-transform gives the optimal end-point conditioned SDE
% \[dX_t= [-\beta X_t + 2\nabla \log h(t,X_t)] dt+\sqrt{2}dW_t, X_0\sim\mathcal{N}(0,1/\beta)\]
% with 
% \[\nabla \log h(t,X_t) = \beta\frac{\exp(-\beta(T-t))}{1-\exp(-2\beta(T-t))}(B-X_t\exp(-\beta(T-t)))\]
% for hitting $X_T=B$ initializing at some $X_0\sim \rho_0$.
% }

To test the efficacy of the proposed schemes in practice, we conduct a series of numerical experiments. Our implementation is in Python, using the JAX library. \footnote{A link to a repository can be found here: \url{https://anonymous.4open.science/r/path_sampling-F861/README.md}}

\subsection{Implementation of Algorithms from Section~\ref{sec:controlled_path_measure}}

We parameterize $\frac{\partial b}{\partial s}$ as a neural net. Concretely, it is chosen to be a MLP with two hidden layers of dimension $20$ and $30$, and ReLU non-linearities. Recalling that at a fixed $s$, $\frac{\partial b}{\partial s}$ takes input $(X_t,t)$ of size $\mathbb{R}^d \times [0,T]$, in the experiments below we train a new MLP at each iteration $s\in [0,1]$ of the algorithm, and use an ADAM optimizer with a learning rate $10^{-3}$ to minimize the specified loss.

%the simplest approach we consider is to encode the input of the MLP as an array of length $d + 1$ representing the position $X_t$ and the time step $t$, in which case each iteration of the algorithm (steps 4-9 of Algorithm \ref{alg:transport_b}) initializes and trains a new MLP. We conjecture that information ought to be shared between iterations, and consider a neural net parametrized also by $s \in [0, 1]$, in which case the same NN is updated at each iteration. TODO: clarify which you use

\subsubsection{Brownian bridge posterior}

We begin with the simple case of sampling from a Brownian bridge pinned at two points $x(0)=-2$ and $x(1)=+2$. Using Algorithm~\ref{alg:transport_b}, we consider transport from the prior $\exp(-I(x))$, set up as a Brownian motion initialized at $x(0)=-2$, i.e., $u^{\text{ref}}=0$. By gradually introducing the endpoint constraint, we design a sequence of measures $\pi_s(x)\propto \exp(-I(x)-sJ(x))$, where $J$ takes the form of $\frac{1}{2\sigma^2}\|x(1)-2\|^2$ with $\sigma=0.1$. As $s$ goes from $0$ to $1$ in $10$ increments of $0.1$, one gradually moves towards the desired Brownian bridge posterior. 

An example of some resulting samples of the posterior Brownian bridge paths are shown in Figure \ref{fig:brownian_bridge} and compared to exact samples simulated from the Brownian bridge SDE \eqref{eqn:BB_SDE} (which we treat as ground truth). In the experiment below, physical time is evolved in increments of $dt=0.01$. At each iteration, the NN is trained for $200$ steps, with $K=500$ paths sampled.

\begin{figure}[H]
    \centering
    \includegraphics[width=0.45\linewidth]{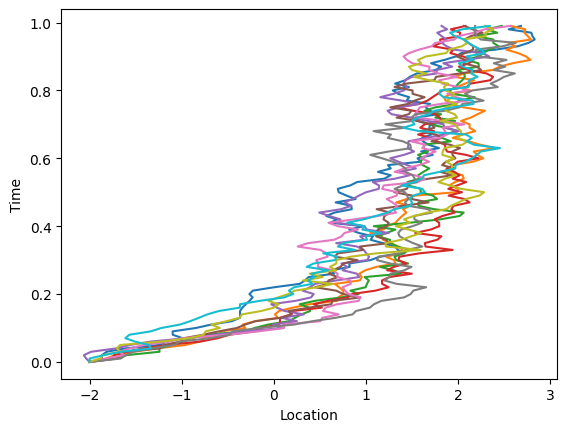}
    \includegraphics[width=0.45\linewidth]{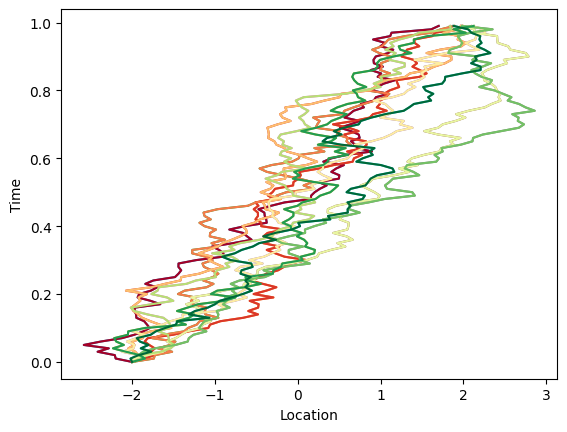}
    \caption{Left: samples from Algorithm~\ref{alg:transport_b} at $s=1$. Right: exact samples from a Brownian bridge SDE.}
\label{fig:brownian_bridge}
\end{figure}

\subsubsection{Transition path sampling}

We next consider a TPS problem in a double well potential $V(x)=5(x^2-1)^2$ with start and end states $A=-1$ and $B=1$ respectively between time $t\in [0,1]$. We conducted two sets of experiments, using two different annealing paths, both based on Algorithm~\ref{alg:transport_b}.

%In this case, we recall that two strategies for applying Algorithm~\ref{alg:transport_b} are available, depending on our choice of $I$ and $J$. The first is described by %TODO: @qijia do we actually describe the first version in the main text?

In the first experiment, we initialize with samples from a double-well potential, treated as the prior $\exp(-I(x))$, and gradually add in the endpoint $B=1$ constraint through $J(x) = \frac{1}{2\sigma^2}\|x(1)-1\|^2$ with $\sigma=0.1$. %, i.e., we relax the delta function constraint on the endpoint $B=1$ to a Gaussian with a standard deviation 
%Results are shown in Figure~\ref{urefprior}. 

In the second experiment, as suggested by the choice of $I$ and $J$ in \eqref{eqn:TPS_anneal}-\eqref{eqn:prior_bb}, samples are initially drawn from a Brownian bridge (therefore the endpoints are already pinned at $A,B$), and we progressively add in the double-well potential to make sure the trajectory probabilities are weighted correctly. Specifically, we expect the trajectories moving fast in between the low density region between the wells to have higher probability under the posterior. Result are shown below in Figure~\ref{urefprior}.

In both cases, $s$ is varied between $[0,1]$ in increment of $\delta_s=0.1$, i.e., $10$ annealing steps. Physical time is evolved with stepsize $dt=0.01$ with $T=1$ as the trajectory length. At each iteration, the NN is trained for $250$ steps in the first experiment and $1000$ in the second, with $K=500$ as the ensemble size in the first experiment, and $K=250$ in the second.

\begin{figure}[H]
    \centering
    \includegraphics[width=0.45\linewidth]{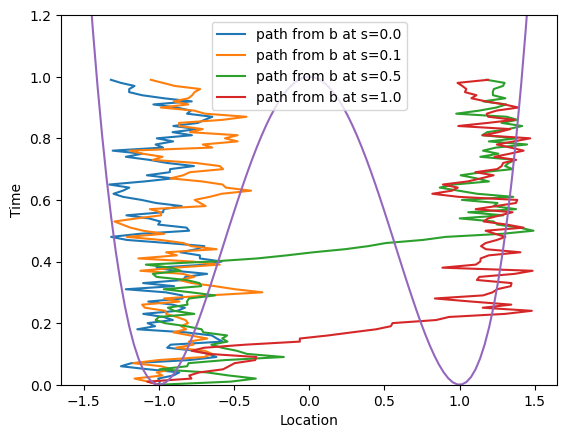}
    \includegraphics[width=0.45\linewidth]{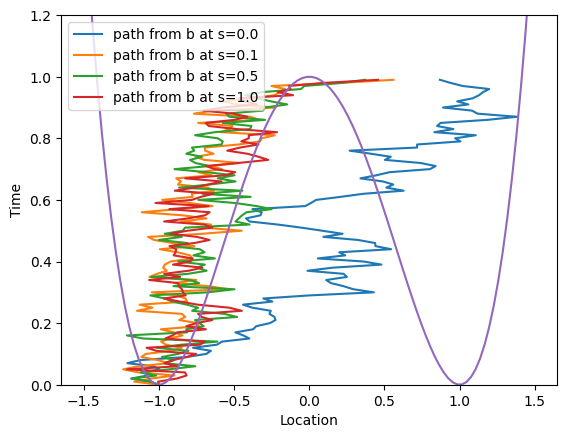}

    \caption{TPS example with paths shown at $s=0.1$, $s=0.5$ and $s=1.0$. Left: evolution from the prior double-well SDE to the posterior with a (soft) constraint at $B=1$. Right: evolution from the Brownian bridge to the posterior. Potential (purple) is scaled for convenience.  MCMC is not used.
    % Left: evolution from the prior double-well SDE to the posterior with a (soft) constraint at $B=1$. Right: evolution from the Brownian bridge to the posterior. Potential (purple) is scaled for convenience. 
    }
    \label{urefprior}
\end{figure}

% Figure \ref{todo} shows a 1 dimensional well described by $V(x)=5(x^2-1)^2$.

We see that when transporting from a Brownian bridge prior, even the paths at small $s$ cross the barrier, while in the other case, the crossing of the barrier happen only when $s$ is sufficiently large.

% We first consider transporting from $u^{\text{ref}}$ (i.e., the prior SDE), as described in \cref{sec:track_b_transport}. 

The left plot in Figure~\ref{urefprior} captures both the constraint and the transition probability very well. With the right plot transporting from the Brownian bridge, we explore the option below with SPDE MCMC turned on between consecutive $s$ updates (c.f. Figure~\ref{fig:spde}). %, we report the result .
% We find that annealing with MCMC is unnecessary. 

\begin{figure}[H]
    \centering   
    \includegraphics[width=0.45\linewidth]{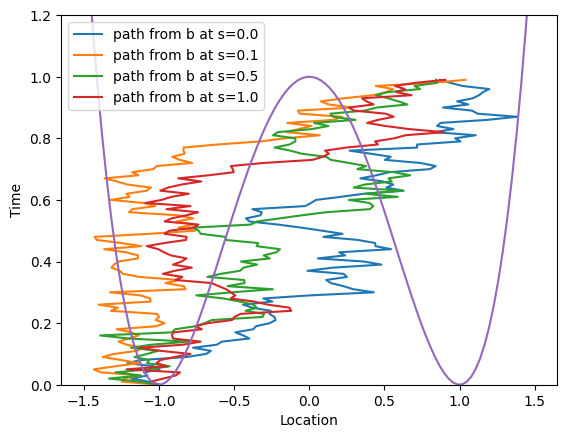}
    \caption{TPS example with paths shown at $s=0.1$, $s=0.5$ and $s=1.0$. Evolution from the Brownian bridge to the posterior with MCMC in between. Potential (purple) is scaled for convenience.}
    \label{fig:spde}
\end{figure}

% Results are shown in figure \ref{urefprior}.

% caption: 
% hyperparams: number of walkers is $K=200$. Number of training steps is $5000$. $T=0.1$, and $\delta t = 0.001$. Without step 7 of \ref{alg:transport_b} vs. with.

% As expected, for $\sigma$ small (e.g. $0.001$), this approach struggles, because the annealing merely ameliorates the strength of the $B=1$ constraint. We now consider the second method discussed in \cref{sec:track_b_transport}, namely using a Brownian bridge as a prior, so that the constraints on $A$ and $B$ are already satisfied by the prior, and annealing adds the effect of the potential on the dynamics.

\subsubsection{Non-equilibrium dynamics}
For the same TPS example as above, we compare what we observe with non-equilibrium dynamics based on Algorithm~\ref{alg:spde}. Here the switching is set to happen very slowly with $s\in[0,1]$ in $100$ steps (i.e., $\delta_s=0.01$). MH adjustment is applied on a discretization of the SPDE after each step of the $s$ update. The annealing path is based on transporting the Brownian Bridge to a pinned double-well posterior at $A=-1, B=1$, with ensemble size $K=100$, stepsize $\delta_t = 10^{-3}$.

We show in Figure~\ref{fig:alg2} the path switch as $s\in [0,1]$, as well as a final set of trajectories at $s=1$. We see quite a lot of variability in the final trajectories, which is an indication of the sampler's performance -- high variance reflects that the trajectories are far from equilibrium and dominated by rare events. % The variance of the weights $w_k=e^{-W^{k,1/\delta_s}}=\text{Var}(\exp(A_1^k))$ is $0.052$, 
%\[=\frac{1}{K}\sum_{k=1}^K \exp(A_1^k)\approx Z_1/Z_0,\] 

\begin{figure}
    \centering
    \includegraphics[width=0.45\linewidth]{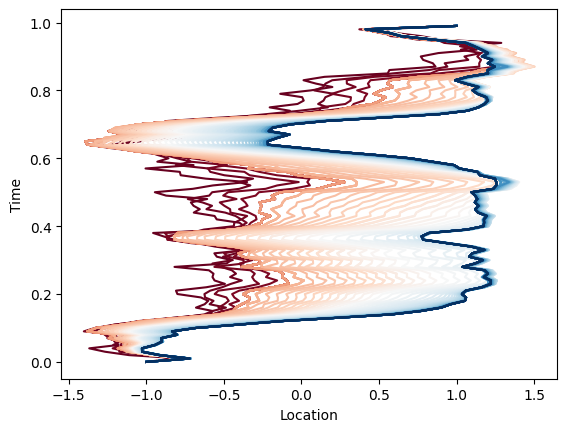}
    \includegraphics[width=0.45\linewidth]{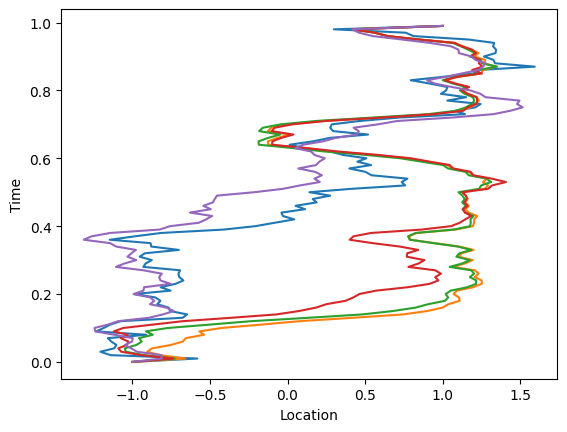}
    \caption{Left: One of the trajectories in the ensemble, at each iteration of the algorithm, going from red ($s=0$) to blue ($s=1$). Right: 5 trajectories from the ensemble at $s=1$.}  % Variance is $0.052$. 
    \label{fig:alg2}
\end{figure}

% \subsection{Double well potential: higher dimensions}

% We include numerical results on double-well potential $V(x)=5(x^2-1)^2$ for TPS with prespecified states $A=-1$ and $B=1$, with 2 different annealing schemes. We start from either $u^{\text{ref}}$ (i.e., the prior SDE) or from the Brownian bridge using the transport method from Section \ref{sec:track_b_transport}.

\subsection{Implementation of Algorithm from Section~\ref{sec:WGF}}

We set up the experiment below so $\nu$ corresponds to the stationary distribution of an Ornstein–Uhlenbeck process
\[dX_t= -\beta X_t  dt+\sqrt{2}dW_t\]
with $\beta=1/4$ -- this is considered to be the prior reference process in this example. Moreover, we initialize the $n=200$ particles as $X_0\sim\mathcal{N}(0,1/\beta)$ at stationarity. The likelihood $J_t$ is chosen so that we impose a soft constraint in the form of $\frac{1}{2\gamma^2}\|x(t)-x_{\text{fixed}}\|^2$ at the midpoint with a big $\gamma$, and the same constraint near the endpoint with a small $\gamma$, at two different pre-specified locations $x_{\text{fixed}}$. 

In the experiment below (Figure~\ref{fig:3dplot}), stepsize $h=0.2$, length of trajectory $T=1$ and the perturbation variance $\sigma^2=0.4^2$ with \#perturbations $m=30$ for the $I(q;\nu)$ estimator. This general Bayesian posterior path example using Algorithm~\ref{alg:jko} is conducted with a $7$-layer MLP parameterizing the pushforward map $\nabla \phi^\theta$. With $x_{\text{fixed}} = -1$ at the midpoint, and $x_{\text{fixed}} = +1$ near the endpoint, the likelihood in this case has the effect of slightly nudging trajectories towards $-1$ in the middle of the path and sharply pin them at $+1$ at the terminal boundary. In between the two observations, we expect the particles to roughly have a Gaussian distribution not too far from stationary. %A simulation is shown in Figure~\ref{fig:3dplot}, where the effect of the perturbation in the middle of the path and the end is clear.

\begin{figure}[H]
    \centering
    \includegraphics[width=0.7\linewidth]{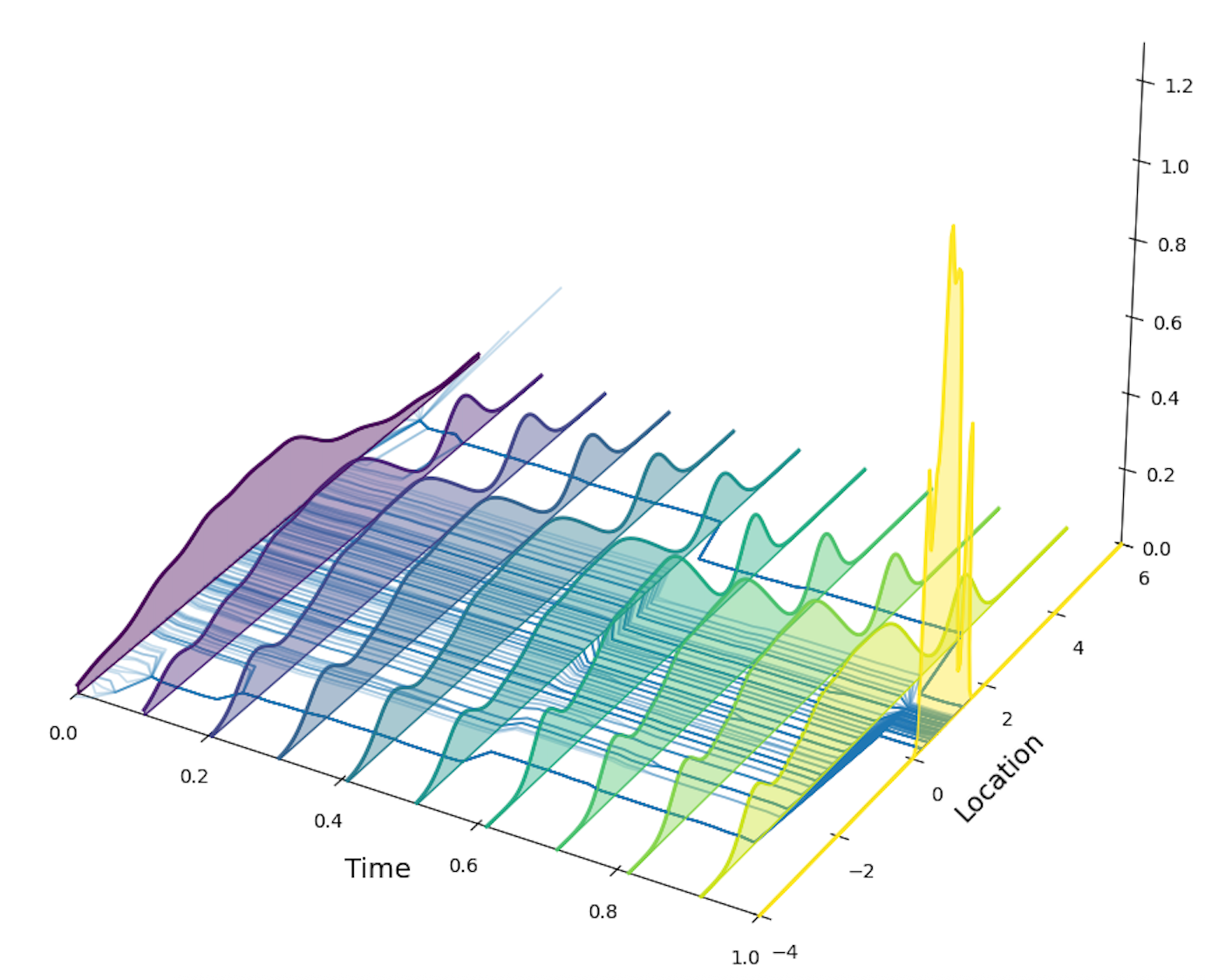}
    \caption{Density evolution of trajectories (reference is set to be an OU process in equilibrium).}
    \label{fig:3dplot}
\end{figure}

% \begin{figure}[H]
%     \centering
%     \includegraphics[width=0.4\linewidth]{double_well_2D.png}
%     \caption{A TPS problem in 2D: the path starts in the [-1,-1] well and finishes in the [1,1] well. Annealing from the SDE prior.}
%     \label{fig:enter-label}
% \end{figure}

% \begin{table}
%   \caption{Sample table title}
%   \label{sample-table}
%   \centering
%   \begin{tabular}{lll}
%     \toprule
%     \multicolumn{2}{c}{Part}                   \\
%     \cmidrule(r){1-2}
%     Name     & Description     & Size ($\mu$m) \\
%     \midrule
%     Dendrite & Input terminal  & $\sim$100     \\
%     Axon     & Output terminal & $\sim$10      \\
%     Soma     & Cell body       & up to $10^6$  \\
%     \bottomrule
%   \end{tabular}
% \end{table}

\end{document}